\pdfoutput=1
\documentclass[letterpaper]{article}
\usepackage[totalwidth=480pt, totalheight=680pt]{geometry}

\usepackage{algorithm}
\usepackage{algorithmicx}
\usepackage{algpseudocode}
\usepackage{subcaption}
\usepackage{enumerate}
\usepackage{bbm}
\usepackage{amsmath}
\usepackage{amssymb}
\usepackage{amsthm}
\usepackage{pifont}
\usepackage{booktabs}
\usepackage{xcolor}
\definecolor{darkgreen}{rgb}{0,0.5,0}
\usepackage{hyperref}
\hypersetup{
    unicode=false,          
    colorlinks=true,        
    linkcolor=red,          
    citecolor=darkgreen,    
    filecolor=magenta,      
    urlcolor=blue           
}
\usepackage[capitalize, nameinlink]{cleveref}

\usepackage{thm-restate}
\theoremstyle{plain}
\newtheorem{theorem}{Theorem}
\newtheorem{proposition}[theorem]{Proposition}
\newtheorem{lemma}[theorem]{Lemma}

\theoremstyle{definition}
\newtheorem{definition}[theorem]{Definition}

\theoremstyle{remark}

\newtheorem{example}[theorem]{Example}

\usepackage{tikz}
\usetikzlibrary{automata, calc, graphs, graphs.standard, shapes, arrows, arrows.meta, positioning, decorations.pathreplacing, decorations.markings, decorations.pathmorphing, fit, matrix, patterns, shapes.misc, tikzmark}

\newcommand{\bA}{\mathbf{A}}
\newcommand{\bB}{\mathbf{B}}
\newcommand{\bC}{\mathbf{C}}

\newcommand{\bG}{\mathbf{G}}

\newcommand{\bQ}{\mathbf{Q}}

\newcommand{\bS}{\mathbf{S}}
\newcommand{\bT}{\mathbf{T}}
\newcommand{\bU}{\mathbf{U}}
\newcommand{\bV}{\mathbf{V}}
\newcommand{\bW}{\mathbf{W}}

\newcommand{\bY}{\mathbf{Y}}
\newcommand{\bZ}{\mathbf{Z}}

\newcommand{\bb}{\mathbf{b}}

\newcommand{\bm}{\mathbf{m}}
\newcommand{\bp}{\mathbf{p}}
\newcommand{\bs}{\mathbf{s}}

\newcommand{\bx}{\mathbf{x}}

\newcommand{\cI}{\mathcal{I}}

\newcommand{\cM}{\mathcal{M}}

\newcommand{\cO}{\mathcal{O}}
\newcommand{\cP}{\mathcal{P}}

\newcommand{\eps}{\varepsilon}
\newcommand{\N}{\mathbb{N}}
\newcommand{\R}{\mathbb{R}}

\newcommand{\wt}{\widetilde}

\DeclareMathOperator*{\argmax}{argmax}
\DeclareMathOperator*{\argmin}{argmin}

\newcommand{\multiset}[1]{\{\!\{ #1 \}\!\}}
\newcommand{\type}{\mathrm{type}}
\newcommand{\obj}{\mathrm{obj}}
\newcommand{\covered}{{\normalfont\textsc{covered}}}
\newcommand{\OPT}{{\normalfont\textsc{OPT}}}
\newcommand{\ALG}{{\normalfont\textsc{ALG}}}
\newcommand{\OurTool}{{\normalfont\textsc{HARP}}}
\newcommand{\OurProblem}{{\normalfont\textsc{MOPGP}}}
\newcommand{\OurAlg}{{\normalfont\textsc{MultistepPlanning}}}
\newcommand{\OurAlgWithAdvice}{{\normalfont\textsc{MultistepPlanningWithAdvice}}}

\title{Optimizing Health Coverage in Ethiopia: A Learning-augmented Approach and Persistent Proportionality Under an Online Budget }
\author{
Davin Choo\thanks{These authors contributed equally}\\
John A. Paulson School of Engineering and Applied Sciences\\
Harvard University\\
\and
Yohai Trabelsi\footnotemark[1]\\
John A. Paulson School of Engineering and Applied Sciences\\
Harvard University\\
\and
Fentabil Getnet\\
National Data Management and Analytics Center for Health\\
Ethiopian Public Health Institute\\
\and
Samson Warkaye Lamma\\
National Data Management and Analytics Center for Health\\
Ethiopian Public Health Institute\\
\and
Wondesen Nigatu\\
Primary Healthcare and Community Engagement Lead Executive Office\\
Ministry of Health, Ethiopia\\
\and
Kasahun Sime\\
Primary Healthcare and Community Engagement Lead Executive Office\\
Ministry of Health, Ethiopia\\
\and
Lisa Matay\\
Department of Global Health and Population\\
Harvard T.H. Chan School of Public Health\\
\and
Milind Tambe\thanks{These authors jointly supervised this work.}\\
John A. Paulson School of Engineering and Applied Sciences\\
Harvard University\\
\and
St\'{e}phane Verguet\footnotemark[2]\\
Department of Global Health and Population\\
Harvard T.H. Chan School of Public Health\\
}

\begin{document}

\maketitle

\begin{abstract}
As part of nationwide efforts aligned with the United Nations' Sustainable Development Goal 3 on Universal Health Coverage, Ethiopia's Ministry of Health is strengthening health posts to expand access to essential healthcare services. However, only a fraction of this health system strengthening effort can be implemented each year due to limited budgets and other competing priorities, thus the need for an optimization framework to guide prioritization across the regions of Ethiopia. In this paper, we develop a tool, Health Access Resource Planner (\OurTool{}), based on a principled decision-support optimization framework for sequential facility planning that aims to maximize population coverage under budget uncertainty while satisfying region-specific proportionality targets at \emph{every} time step. We then propose two algorithms: (i) a learning-augmented approach that improves upon expert recommendations at any single-step; and (ii) a greedy algorithm for multi-step planning, both with strong worst-case approximation estimation. In collaboration with the Ethiopian Public Health Institute and Ministry of Health, we demonstrated the empirical efficacy of our method on three regions across various planning scenarios.    
\end{abstract}

\section{Introduction}
\label{sec:intro}

Ethiopia, the second most populous country in Africa, is home to over 130 million people across 12 regional states and two chartered cities, further divided into zones, woredas (districts), and kebeles (the smallest administrative units) \cite{unfpa2022ethiopia,worldometers2022}.
Roughly 76\% of the population lives in rural areas, where access to healthcare is limited compared to urban settings \cite{hendrix2023estimated}.
While the physical accessibility of health facilities has significantly improved, the issue is ensuring the quality of essential healthcare service delivery.

The Ethiopian healthcare system follows a three-tier structure comprising Primary Health Care Units (PHCUs), general hospitals, and specialized hospitals.
PHCUs, which consist of primary hospitals, health centers, and health posts and are responsible for 70–80\% of essential health services, have been expanded under the Ministry of Health’s (MOH) flagship Health Extension Program (HEP), launched in 2003/04 \cite{wang2016ethiopia}.

Unfortunately, persistent challenges remain as the evolution of the HEP is not aligned with the ever increasing demand of the community to get a comprehensive essential healthcare services (health promotion-rehabilitative care) at the health post.
This resulted in inadequate staffing and equipment, and declining and/or stagnating performance in maternal and child health indicators \cite{Teklu2020}.

To address these gaps, the MOH introduced the HEP Optimization Roadmap (2020--2035) \cite{moh2020roadmap}.
A central reform is the reclassification of health posts into three categories: mixed, basic, and comprehensive. 
Comprehensive health posts are designed to deliver a broader suite of services, including childbirth, postnatal care, and chronic disease treatment, and require substantial investments in infrastructure, staffing, and medical supplies.

The roadmap calls for the construction (or upgrading) of over 2,000 comprehensive health posts, with an estimated cost of USD 300k–400k per facility excluding staffing and equipment.
As of now, 260 are operational and 183 are under construction.
While initial decisions used WHO's AccessMod platform (\url{https://www.accessmod.org/}), they did not necessarily incorporate region-specific prioritization preferences (see below) or account for geographic and demographic factors known to drive service disparities \cite{getnet2025inequalities}.

\begin{figure}[htb]
     \centering
     \includegraphics[width=0.49\linewidth]{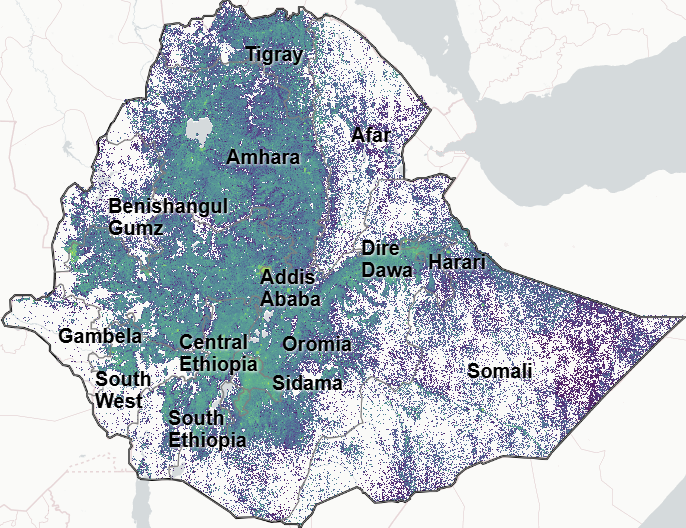}
     \includegraphics[width=0.49\linewidth]{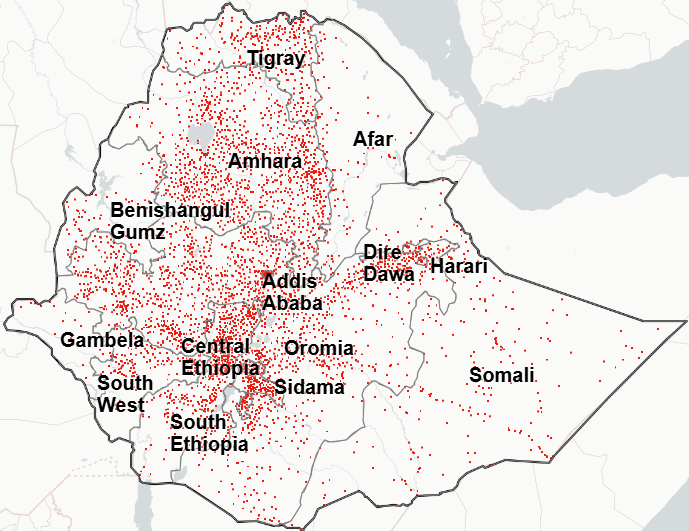}
     \caption{Map of Ethiopia overlaid with population estimates for 2026 (log scale). Left: Brighter yellow areas indicate higher population density. Right: Locations of primary hospitals, health centers, and comprehensive health posts that can provide essential health services are highlighted in red.}
     \label{fig:intro-figure}
\end{figure}

Given limited resources and growing health demands, there is a pressing need for a strategic, data-driven approach to guide equitable facility placement (e.g., where to build the next comprehensive health post given limited annual budget to augment the existing network of healthcare facilities).
This requires integrating diverse datasets, including geolocation of existing facilities, high-resolution population estimates, topographic and official district administrative shapefiles, and health service coverage indicators like skilled birth attendance and postnatal care.
Planning is further complicated by uncertainty in budgets and population forecasts, and the need to ensure fairness at every decision point.
Different regions may also impose varying distributional preferences: some may prioritize population density while others may target poorer districts.

\paragraph{Our contributions.}
In collaboration with the Ethiopian Public Health Institute and Ministry of Health, we formalize the problem and develop a Health Access Resource Planner
(\OurTool{}) tool to help regional planners navigate these challenges and customize allocation strategies to local priorities.

\begin{enumerate}
    \item In~\cref{sec:formulation}, we propose and study the new abstract problem of \underline{m}aximizing a non-decreasing submodular function under \underline{o}nline \underline{p}artition and \underline{g}lobal \underline{p}roportional constraints (\OurProblem{}), which captures the real-world challenge of prioritizing health facility upgrades in Ethiopia.
    Crucially, our framework accommodates region-specific distributional goals by encoding proportionality constraints, making it a flexible, shared decision-support tool for planners with possibly differing regional priorities.
    \item We design algorithms to solve \OurProblem{} with provable worst-case guarantees.
    These include a learning-augmented algorithm that integrates expert-provided selections, supporting regional planning workflows that start with partial plans and seek principled refinement.
    \item In ~\cref{sec:experiments}, \cref{sec:appendix-experiment-sidama} and in    \cref{sec:appendix-experiment}, we evaluate the \OurTool{} tool on some regions of Ethiopia using two possible prioritization rules.
    We show that our approach successfully prioritizes coverage according to the specified criteria while maintaining reasonable overall coverage.
\end{enumerate}
\OurTool{} enables regional planners to solve their specific instance of \OurProblem{} using a unified, principled approach.
While our empirical results offer insight into practical performance, we emphasize that all findings are intended as a proof of concept and not official policy recommendations.
\footnote{}
\section{Problem formulation}
\label{sec:formulation}

\textbf{Notation.}
We use lowercase letters for scalars, boldface for sets and sequences, and calligraphic letters for higher-order objects (e.g., sets of sets, graphs, matroids).
For a set $\bA$, let $|\bA|$ denote its cardinality, and $2^{\bA}$ its power set.
The disjoint union of disjoint sets $\bA$ and $\bB$ is denoted $\bA \uplus \bB$.
We use $\R_{\geq 0}$ for non-negative reals, $\N$ for natural numbers, $\N_{>0} = \N \setminus \{0\}$, and $[n] = \{1, \ldots, n\}$ for any $n \in \N_{>0}$.
For any sequence $\bs \in [r]^b$ and type $q \in [r]$, let $\#(q, \bs)$ denote the number of times $q$ appear in $\bs$, e.g., for $\bs = (1,2,1,3) \in [3]^4$, we have $\#(1, \bs) = 2$, $\#(2, \bs) = 1$, and $\#(3, \bs) = 1$.
\medskip

In the following, we will formally define \OurProblem{} in full generality and then contextualize to how it assists in facility location within any particular region of Ethiopia.

We model \OurProblem{} as a sequential subset selection task, where the goal is to maximize a non-decreasing submodular function (see \cref{sec:related-work} for a formal definition) subject to two constraints: (i) an \emph{online budget constraint} modeling incremental budget arrivals, and (ii) a \emph{global proportional constraint} enforcing type-level proportionality.
Let $\bV = \bV^{(1)} \uplus \ldots \uplus \bV^{(h)}$ denote the ground set partitioned over $h$ rounds.
Each element $e \in \bV$ is associated with a type $\type(e) \in [r]$, inducing an second partitioning $\bV = \bT_1 \uplus \ldots \uplus \bT_r$, where $\bT_q = \{ e \in \bV : \type(e) = q \}$.
At each round $t \in [h]$ (a year), a budget $b^{(t)} \leq |\bV^{(t)}|$ is revealed and we must irrevocably select a subset $\bS^{(t)} \subseteq \bV^{(t)}$ of size $|\bS^{(t)}| \leq b^{(t)}$.
We define the cumulative budget and selection up to round $t$ as $b^{(1:t)} = \sum_{\tau=1}^t b^{(\tau)}$ and $\bS^{(1:t)} = \uplus_{\tau=1}^t \bS^{(\tau)}$ respectively.
The objective is to maximize $f: 2^{\bV} \to \R$, a non-decreasing submodular function over the final selection $\bS = \bS^{(1:h)} \subseteq \bV$, subject to the follow two constraints:
\begin{enumerate}
    \item \textbf{Online budget constraint}: $|\bS^{(t)}| \leq b^{(t)}$ for all $t \in [h]$.
    \item \textbf{Global proportional constraint}: For given proportions $p_1, \ldots, p_r \in [0,1]$ with $\sum_{q=1}^r p_q \leq 1$, the selection should satisfy $|\bS \cap \bT_q| \geq p_q \cdot |\bS|$ for all types $q \in [r]$.
\end{enumerate}

Unfortunately, it may not be feasible to always meet proportional targets.
For instance, if $b^{(1)} = 1$ and we require each of $r \geq 2$ types to have at least $\frac{1}{10r}$ proportion, then no single-element selection in the first time step can meet this condition.
As such, we define the notion of \emph{satisfaction ratio} as
$
\alpha_q(\bS) = \frac{|\bS \cap \bT_q|}{p_q \cdot |\bS|}
$
where
$
\alpha_{\min}(\bS) = \min_{q \in [r]} \alpha_q(\bS)
$
and define the best possible such ratio under budget $b$ as $\beta(b) = \max_{\bA \subseteq \bV, |\bA| = b} \alpha_{\min}(\bA)$.
Note that the proportional constraints on all types \emph{can} be satisfied when $\beta(b) \geq 1$.

To ensure long-term proportional balance, we insist that any feasible solution is a type feasible selection.

\begin{definition}[Type feasible selection]
\label{def:type-feasible-selection}
A subset $\bS^{(1:h)} \subseteq \bV$ of size $b^{(1:h)}$ is \emph{type feasible} if, for all time steps $t \in [h]$:
\begin{enumerate}
    \item $\alpha_{\min}(\bS^{(1:t)}) = \beta(b^{(1:t)})$
    \item Among all such sets, it minimizes $|\bQ(\bS^{(1:t)})|$ where
    \[
    \bQ(\bS^{(1:t)}) = \bigg\{ q \in [r] : \alpha_q(\bS^{(1:t)}) = \beta(|\bS^{(1:t)}|)
    \text{ and } \bV^{(t)} \cap \bT_q \neq \emptyset \bigg\}
    \]

\end{enumerate}
\end{definition}

Intuitively, the set $\bQ(\bS^{(1:t)})$ contains all types with the minimum satisfaction ratio.
Since improving the satisfication ratio $\alpha_{\min}(\cdot)$ requires one to select at least one at least element of each type in $\bQ(\cdot)$, minimizing $|\bQ(\cdot)|$ serves to hedge against potential budget uncertainty.

We are now ready to formally define \OurProblem{}.

\begin{definition}[\OurProblem{}]
Given a non-decreasing submodular function $f: 2^\bV \rightarrow \R$, time horizon $h \in \N$, number of types $r \in \N_{>0}$, ground set $\bV = \bV^{(1)} \uplus \ldots \uplus \bV^{(h)} = \bT_1 \uplus \ldots \uplus \bT_r$, online budgets $\bb = \{ b^{(1)}, \ldots, b^{(h)} \}$, and type proportions $\bp = \{ p_1, \ldots, p_r \}$, the \OurProblem{}$(f, h, r, \bV, \bb, \bp)$ problem seeks to compute a type feasible selection $\bS \subseteq \bV$ such that $|\bS^{(t)}| \leq b^{(t)}$ for all time steps $t \in [h]$.
\end{definition}

Observe that \OurProblem{} is a generalization of the well-studied NP-hard maximum coverage problem for a suitably defined set function $f$ when $t = 1$, $r = 1$, and $p_1 = 0$ (i.e., no proportionality constraints).
As such, one cannot expect to be able to compute an optimum selection $\bS^* \subseteq \bV$ in polynomial time in general.
Instead, we will design and analyze efficient approximation algorithms that can compute some selection $\bS \subseteq \bV$ such that $f(\bS) \geq \alpha \cdot f(\bS^*)$ for some approximation ratio $0 \leq \alpha \leq 1$.

\paragraph{Contextualizing \OurProblem{} to the Ethiopian setting.}
For our problem, we use the 1km-by-1km grid population forecasts for each region in Ethiopia \cite{worldpop_eth_2015_2030}, as well as geographical information that enables computation of travel time \cite{weiss2020global}.

Let $\bU$ denote the space of grid cells considered.
For each year $t \in [h]$, the set $\bV^{(t)}$ correspond to all grid cells in the region for which we can build a comprehensive health post while types correspond to the districts within the different regions in Ethiopia.
Note that $\bV^{(1)}, \ldots, \bV^{(h)}$ are all defined over the same set of grid cells $\bU$ but correspond to building a facility at different time steps.
For any facility located at grid cell $c \in \bU$, let $\covered{}(c) \subseteq \bU$ be the set of grid cells that are reachable from $c$ within 2-hours of travel time, which depends on the topological characteristics of the terrain and may vary across locations and directions.
For a set of facility locations (i.e., subset of grid cells) $\bS$, we write $\covered{}(\bS) = \bigcup_{c \in \bS} \covered{}(c) \subseteq \bU$ to mean the union of the covered cells.

As the planning process occurs over a multi-year horizon where facilities continue to provide coverage in all future time steps once it has been built, we define the objective function $f$ is the number of people who can access a comprehensive health post within a walking time of at most two hours under the proposed selection.
To be precise,
\begin{equation}
\label{eq:ethiopia-objective-f}
f(\bS) = \sum_{t=1}^h \sum_{c \in \covered{}(\bS^{(1:t)})} w^{(t)}_c
\end{equation}
where $w^{(t)}_c$ is the population forecast for year $t$ for grid cell $c$.
We formally show that the above objective \cref{eq:ethiopia-objective-f} is submodular in \cref{sec:appendix-ethiopia-objective-f}.
\cref{tab:ethiopia-mapping} summarizes the mapping while \cref{example:formulation} illustrates our formulation visually.

\begin{table}[htb]
\centering
\begin{tabular}{@{}cc@{}}
\toprule
\textbf{Parameters for a region} & \textbf{Component in \OurProblem{}}\\
\midrule
5-year planning horizon & $h = 5$ time steps \\
Districts within a region & $r$ types \\
Annual building budget & Online budgets $b^{(t)}$ \\
Location, at year $t$ & The set $\bV^{(t)}$ \\
Posts built for year $t$ & $\bS^{(t)} \subseteq \bV^{(t)}$ with $|\bS^{(t)}| \leq b^{(t)}$\\
2-hour accessibility goal & Determines $c \in \covered{}(\bS^{(1:t)})$ \\
Forecast for cell $c$ in year $t$ & Population weight $w^{(t)}_c$ \\
Coverage objective & $f$ defined in \cref{eq:ethiopia-objective-f} \\
District prioritization & Proportions $p_1, \ldots, p_r$ \\
Persistent proportionality & Type feasibility in \cref{def:type-feasible-selection} \\
\bottomrule
\end{tabular}
\caption{Mapping the Ethiopian context to \OurProblem{}.}
\label{tab:ethiopia-mapping}
\end{table}
\begin{figure}[htb]
    \centering
    \includegraphics[width=\linewidth]{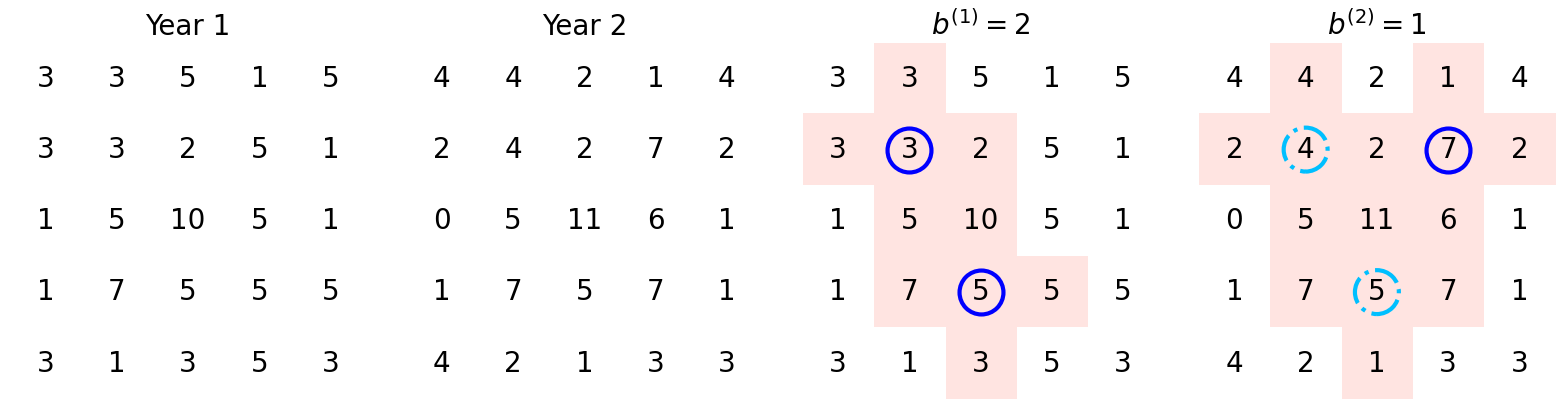}
    \caption{An example illustrating \OurProblem{} notation defined in Section~\ref{sec:formulation}.
    Health post locations are indicated by {\color{blue}blue} circles.
    Covered cells are colored.
    Observe that the two facilities built in Year 1 continue to provide coverage in Year 2.
    }
    \label{fig:example-formulation}
\end{figure}

\begin{example}
\label{example:formulation}
Consider the $5 \times 5$ population grid in \cref{fig:example-formulation} over $h = 2$ time steps, with $b^{(1)} = 2$ and $b^{(2)} = 1$.
Suppose a facility at any cell provides coverage for any cell that is of Manhattan distance 1 from it.
If we build two health posts at coordinate $\bS^{(1)} = \{(2,2), (4,3)\}$ at the first time step and another at $\bS^{(2)} = \{(2,4)\}$ at the second time step, then we see that $f(\bS) = f(\bS^{(1)} \uplus \bS^{(2)}) = 110$, where we obtain a coverage of $46$ from Year $1$ and $64$ in Year $2$.
Observe that the population can change over time and health posts continue to provide coverage in future time steps once built.
\end{example}

\section{Key concepts and related work}
\label{sec:related-work}

After relating our formulation of \OurProblem{} to the known concepts of submodular optimization and matroid theory, we review related work on optimization under these structures, highlighting limitations of existing methods in addressing the sequential and proportional constraints of \OurProblem{}.
Finally, we briefly introduce learning-augmented algorithms.
For more related work, see \cref{sec:appendix-related-work}

\medskip
\noindent
\textbf{Submodular Functions.}
Let $\bV$ be a finite ground set (e.g., the pairs (location, year) for building health posts) and $f : 2^{\bV} \to \R$ be a set function, e.g., a function that assigns a coverage value to any given set of locations where health posts are built.
We say that $f$ is non-decreasing if $f(\bA) \leq f(\bB)$ for all $\bA \subseteq \bB \subseteq \bV$, and $f$ is submodular if for all $\bA \subseteq \bB \subseteq \bV$ and $e \in \bV \setminus \bB$, $f(\bA \cup \{e\}) - f(\bA) \geq f(\bB \cup \{e\}) - f(\bB)$.
It is known that a greedy selection gives a $1 - 1/e$ approximation for maximizing non-decreasing submodular functions under a cardinality constraint \cite{nemhauser1978analysis}.
For notational convenience, we will write the marginal gain function $g$ with respect to $f$ as $g(\bA, \bB) = f(\bA \uplus \bB) - f(\bA)$ for any two disjoint subsets $\bA, \bB \subseteq \bV$.

\medskip
\noindent
\textbf{Matroids.}
Matroids encompass a wide range of constraints that enable efficient optimization algorithms that admit provable approximation guarantees.
In this work, we are particularly interested in the special class of partition matroid \cite{recski1973partitional} which has applications like load balancing, scheduling, and constrained subset selection.

\begin{definition}[Partition matroid]
Let $\bV = \bV_1 \uplus \ldots \uplus \bV_r$ be a partition of $\bV$ and $x_1, \ldots, x_r \in \N$ be upper bounds.
Then, the independence set $\cI$ of a partition matroid $\cM = (\bV, \cI)$ is defined as $\cI = \{ \bS \subseteq \bV : |\bS \cap \bV_q| \leq x_q, \forall q \in [r] \}$.
\end{definition}

In our Ethiopian setting, budget and proportionality constraints each induce a partition matroid over the set of candidate facilities, e.g., budget at time step $t$ imposes the constraint that any valid selection $\bS$ satisfies $|\bS^{(t)} \cap \bV^{(t)}| \leq b^{(t)}$.

\medskip
\noindent
\textbf{Online maximization under matroid constraints.}
A closely related problem to \OurProblem{} is that of Problem (1.6) of \cite{fisher1978analysis}, which we present using our notation in \cref{def:fisher-problem-definition}.

\begin{definition}[Maximizing a non-decreasing submodular function subject to $k$ matroid intersections over $h$ rounds]
\label{def:fisher-problem-definition}
Let $\bV = \bV^{(1)} \uplus \ldots \uplus \bV^{(h)}$ be a set of $h$ disjoint ground sets, $f: 2^{\bV} \to \R$ be a non-decreasing submodular set function, and $\cM_1 = (\bV, \cI_1), \ldots, \cM_k = (\bV, \cI_k)$ be $k$ matroids over $\bV$.
The goal is to find a subset $\bS \subseteq \bV$ that maximizes $f(\bS)$ subject to $\bS \cap \bV^{(t)} \in \cI_1 \cap \ldots \cap \cI_k$ for all $t \in [h]$.
\end{definition}

If budget uncertainty were absent and only final-round proportional constraints mattered, the problem could be modeled using two partition matroids ($k = 2$), allowing the local greedy algorithm of \cite{fisher1978analysis} to achieve a $1/(k+1) = 1/3$ approximation.
However, \OurProblem{} differs fundamentally from \cref{def:fisher-problem-definition} in that matroidal constraints -- driven by online budget arrivals -- are revealed incrementally over time.
This renders direct application of their method infeasible.

For the single-matroid ($k = 1$) case, \cite{calinescu2011maximizing} achieve a $1 - 1/e$ approximation via continuous greedy methods, while \cite{goundan2007revisiting} show that local greedy retains a $1/(\alpha + 1)$ approximation even when optimizing a proxy function $g$ that approximates the true objective $f$ within a factor of $\alpha \geq 1$.

To address evolving constraints, one may consider online optimization frameworks.
The submodular secretary problem \cite{babaioff2007matroids,bateni2013submodular} assumes fixed matroids and randomly arriving elements, whereas our setting reverses this: the ground set is fixed, but feasibility constraints (budgets) are revealed over time.
Recent work by \cite{cristi2024online,santiago2025constant} explores online feasibility arrival but focuses on linear objectives and does not support global, type-based fairness constraints.

Finally, fairness-aware submodular optimization \cite{celis2019classification, tsang2019group} has addressed proportionality constraints, but these methods typically require full input access and use continuous relaxations or local exchanges, making them unsuitable for our setting of irrevocable, sequential decisions under uncertainty.

\medskip
\noindent
\textbf{Representation-constrained coverage and online constraints}
Representation constraints have been studied in coverage problems.
For example, \cite{asudeh2023maximizing} consider a max-$k$-coverage problem with group fairness constraints but enforce rigid group-wise coverage parity, which can lead to infeasibility.
In contrast, our model allows flexible type-level bounds, maintaining feasibility and realism.
Recently, (\cite{cristi2024online}; \cite{santiago2025constant}) explored online problems where constraints, rather than elements, arrive over time, similar to our evolving feasibility model.
However, they focus on maximizing simple linear objectives and do not directly extend to submodular objectives or type-based distributional constraints.

\medskip
\noindent
\textbf{Learning-augmented algorithms.}
In practice, regional governments may already have existing facility selection plans (for any single fixed time step with budget $b$) 
that are constructed using domain knowledge, heuristics, or political considerations.
Such an advice selection often encodes valuable local preferences but may not be optimal under the formal objective.
Learning-augmented algorithms provide a formal framework to improve decision quality by incorporating advice or predictions, while maintaining worst-case guarantees when the advice is poor \cite{lykouris2021competitive}.
The goal is to design algorithms that are both \emph{consistent} (optimal when the advice is perfect) and \emph{robust} (competitive with the best advice-free baseline), ideally with performance degrading gracefully as advice quality worsens.
See \cite{mitzenmacher2022algorithms} for a survey and \cref{sec:appendix-related-work} for further examples.
Closest to us are the works of \cite{liu2024predicted,agarwal2024learning}, which aim to improve update times in dynamic settings.
In contrast, we aim to improve the approximation ratio achievable by any polynomial-time algorithm in a static selection setting.

\section{Algorithmic contributions}
\label{sec:algorithmic-contributions}

We present a learning-augmented\footnote{While ``learning-augmented'' originated from using machine-learned predictions, the framework more broadly encompasses any source of advice, such as expert-curated selections in our setting.} approach for optimizing facility locations at any single time step in \cref{sec:single-step} and analyze a greedy approach for multi-step planning under uncertainty in \cref{sec:multi-step}.
While our learning-augmented approach applies at every decision point, we only provide guarantees for a single fixed time step.\footnote{Single-step guarantees offer interpretable, modular assurances and are especially relevant in dynamic policy settings, where decisions are made incrementally over multiple time steps.}
In ~\cref{sec:impossibility}, we highlight the importance of \cref{def:type-feasible-selection} by showing that one cannot achieve non-trivial  guarantees for \OurProblem{} under online budgets without additional structure on the problem instance.
All proofs are deferred to the appendix.

\subsection{Learning-augmented single-step planning}
\label{sec:single-step}

Our goal in this section is to design a learning-augmented algorithm that uses such advice to produce selections from $\bV$ that are provably no worse and potentially better.
Here, we present the algorithm for the simple case of $r = 1$ type (\cref{alg:LA-one-type}) and show how to extend it to the more general partition matroid setting with $r \geq 1$ in \cref{sec:appendix-single-step}.

A key advantage of the learning-augmented approach is its practicality in real-world decision-making environments.
Rather than discarding or overriding expert input, our method builds on it --- treating existing recommendations as a starting point and refining them in a principled way.
This makes it particularly attractive to policymakers and planners, as it respects domain expertise accumulated over years of fieldwork, local knowledge, and institutional memory.
Importantly, our approach offers formal performance guarantees: it matches the expert plan when it is already optimal, and improves upon it when possible, without ever doing worse than the best advice-free alternative.
This balance between respecting human judgment and augmenting it with algorithmic rigor makes the method especially compelling in collaborative, policy-driven settings.

At the high-level, our algorithmic approach is to augment subsets of the advice selection with greedy selections\footnote{\OurProblem{} generalizes the NP-hard problem of maximum coverage, and greedy selections attain tight approximation guarantees.}
For any subset $\bA' \subseteq \bA$, we fill the remaining $b - |\bA'|$ slots by greedily selecting elements from $\bV \setminus \bA'$ to maximize marginal gain with respect to the non-decreasing submodular set function $g(\cdot, \bA') = f(\cdot \cup \bA') - f(\bA')$.
That is, we greedily choose elements conditioned on already chosen $\bA'$.

\begin{algorithm}[htb]
\caption{Learning-augmented single-step of \OurProblem{}}
\label{alg:LA-one-type}
\begin{algorithmic}[1]
\Statex \textbf{Input}: Elements $\bV$, non-decreasing submodular set function $f: \bV \to \R$ to maximize for, budget $b \geq 1$, advice selection $\bA \subseteq \bV$ of size $|\bA| = b$
\Statex \textbf{Output}: A selection $\bU \subseteq \bV$
\State Define subsets $\bA_0, \bA_1, \ldots, \bA_b$ of $\bA$ where $|\bA_i| = i$ for sizes $i \in \{0, 1, \ldots, b\}$ \Comment{$\bA_0 = \emptyset$ and $\bA_b = \bA$}
\For{$i \in \{0, 1, \ldots, b\}$}
    \State Initialize $\bB_i = \emptyset$
    \While{$|\bB_i| < b - i$}
        \State Add $e^* = \argmax\limits_{e \in \bV \setminus (\bA_i \cup \bB_i)} g(\bA_i \cup \bB_i, \{e\})$ to $\bB_i$
    \EndWhile
    \State Define $\bU_i = \bA_i \uplus \bB_i$ \Comment{$|\bU_i| = |\bA_i| + |\bB_i| = b$}
\EndFor
\State \Return $\bU_{i^*}$, where $i^* = \argmax_{i \in \{0, 1, \ldots, b\}} f(\bU_i)$
\end{algorithmic}
\end{algorithm}

\begin{restatable}{theorem}{LAonetypeguarantee}
\label{thm:LA-one-type-guarantee}
Consider \cref{alg:LA-one-type} and let $\OPT \subseteq \bV$ be an optimal selection.
Then,
\[
f(\bU)
\geq \max_{i \in \{0, 1, \ldots, b\}} f(\bA_i)
+ \frac{1 - \frac{1}{e}}{\left\lceil \frac{|\OPT \setminus \bA_i|}{b-i} \right\rceil} \cdot g(\OPT, \bA_i)
\]
\end{restatable}

\begin{figure}[htb]
    \centering
    \includegraphics[width=\linewidth]{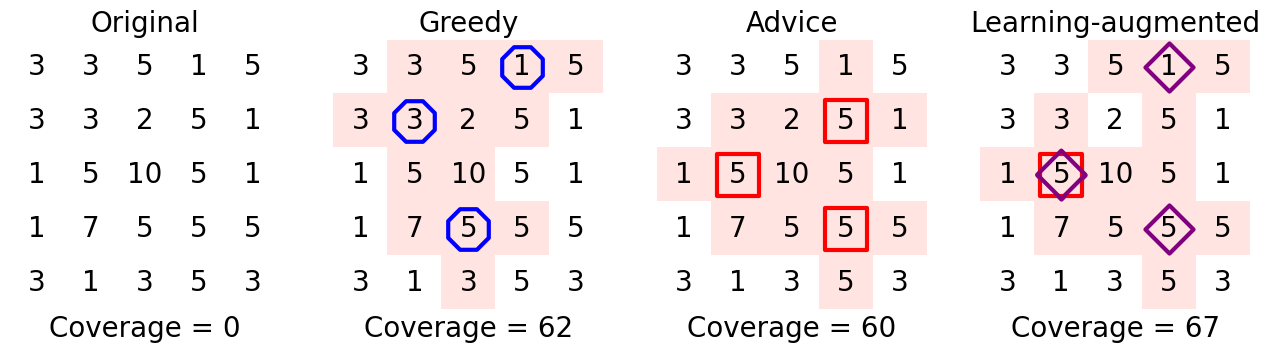}
    \caption{Using the facility at $(2,1)$ as the partial advice, \cref{alg:LA-one-type} produces a selection with improved coverage.
    Note: The top-left cell is coordinate $(0,0)$.}
    \label{fig:example-learning-augmented}
\end{figure}

\cref{alg:LA-one-type} is essentially an algorithm for size-constrained non-decreasing submodular maximization.
Observe that $\bU_0$ is the advice-free greedy selection $\bS$, and $\bU_b$ is the advice selection $\bA$, so we always have $f(\bU_{i^*}) \geq \max\{ f(\bU_0), f(\bA) \}$ and thus the output $\bU_{i^*} \subseteq \bV$ of our learning-augmented algorithm is at least $(1 - \frac{1}{e})$-robust and $1$-consistent.
This is because $f(\bU_{i^*}) \geq f(\bU_0) \geq (1 - \frac{1}{e}) \cdot f(\OPT)$ by submodularity of $f$, and $f(\bU_{i^*}) = f(\OPT)$ when $f(\bA) = f(\OPT)$.
Our approach of \cref{thm:LA-one-type-guarantee} naturally extends to $r \geq 1$ types for type-decomposable objective functions, but with the base guarantee of $f(\bU_0) \geq (1 - \frac{1}{e}) \cdot f(\OPT)$ degrading to $f(\bU_0) \geq \frac{1}{2} \cdot f(\OPT)$ due to the matroidal constraints; see \cref{sec:appendix-single-step} for details.

As a remark, we would like to point out that there is nothing inherently special about the choice of $\bA_0, \ldots, \bA_b$.
From the proof of \cref{thm:LA-one-type-guarantee} in the appendix, one can see that \cref{alg:LA-one-type} can be modified to work with any polynomial number of subsets of $\bA$ while still being polynomial time.

\begin{example}
\label{example:learning-augmented}
\cref{fig:example-learning-augmented} shows three possible selections of $b = 3$ facilities on the first grid of \cref{example:formulation}.
As before, suppose a facility at any cell provides coverage for any cell that is of Manhattan distance 1 from it.
Greedily selecting based on marginal gain produces the selection $\bG = [(3, 2), (0, 3), (1, 1)]$ with $f(\bG) = 62$.
Meanwhile, an advice selection of $\bA = [(2, 1), (1, 3), (3, 3)]$ has $f(\bA) = 60$.
Running \cref{alg:LA-one-type} with the advice facility at coordinate $(2, 1)$ produces the selection $\bU = [(2, 1), (3, 3), (0, 3)]$ with $f(\bU) = 67$, which also happens to be an optimum selection.
\end{example}

\subsection{Multi-step planning under budget uncertainty}
\label{sec:multi-step}

We now show that a variant of the local greedy algorithm from \cite{fisher1978analysis}, adapted to satisfy type feasibility, is both practical and provably robust.
A key difficulty in solving \OurProblem{} lies in the abstract nature of type feasibility: \cref{def:type-feasible-selection} specifies high-level desiderata but does not provide a directly actionable definition.
While the objective $f$ is non-decreasing and submodular, the proportional constraints across types introduce combinatorial complexity that does not decompose neatly.

To make the notion of type feasibility operational, we introduce a fixed tie-breaking rule $\sigma: [r] \to [r]$ over types and restrict our attention to a structured subclass of solutions, which we call $\sigma$-type feasible selections.

\begin{definition}[$\sigma$-type feasible selection]
\label{def:sigma-type-feasible-selection}
Let $\sigma: [r] \to [r]$ be an arbitrary total preference ordering over types.
A subset $\bS \subseteq \bV$ of size $b$ is said to be $\sigma$-type feasible if it is type feasible \emph{and} for $\bQ(\bS^{(1:t)})$, favors types according to $\sigma$: if $\sigma(q) < \sigma(q')$, then type $q$ is preferred over $q'$.
\end{definition}

This reduction justifies restricting our attention to the class of $\sigma$-type feasible solutions: they are not only easier to operationalize algorithmically, but also provably approximate any type-feasible solution up to a factor that depends on the granularity of type availability.

Our next result tells us that any optimum $\sigma$-type feasible selection $\OPT_\sigma$ is competitive against any optimum type feasible selection $\OPT$ if at least $k$ elements of each type is chosen per round when applying \OurAlg{}.

\begin{restatable}{theorem}{nosigma}
\label{thm:no-sigma}
Let $\OPT$ be an optimal type-feasible selection and let $\OPT_{\sigma}$ be an optimal $\sigma$-type feasible selection.
Fix an integer $k \in \N_{> 0}$.
If $|\OPT_{\sigma} \cap \bV^{(t)} \cap \bT_q| \geq k$ for all $t \in [h]$ and $q \in [r]$, then $f(\OPT) \leq \frac{k+1}{k} \cdot f(\OPT_{\sigma})$.
\end{restatable}

Note that we typically have $k \geq 1$: the Ethiopian Ministry of Health is currently building $2,000$ comprehensive health posts across $\sim$$670$ rural districts \cite{OCHA_Ethiopia_AdminBoundaries_2025}, and we expect more facilities to be built in the future.

Our algorithm \OurAlg{} is presented in \cref{alg:ouralg}.
At each round $t \in [h]$, \OurAlg{} determines how many comprehensive health posts $x^{(t)}_q \in \N$ from each district $q \in [r]$ should be selected (line 3), based on the desired global proportion and elements selected so far in $\bS^{(1:t-1)}$.
Note that $\sum_{q \in [r]} x^{(t)}_q = b^{(t)} \leq |\bV^{(t)}|$.
These type-specific quotas define a partition matroid (line 4), over which the algorithm performs greedy selection (lines 5-7) by iteratively adding the element with highest marginal gain while respecting the matroid.
\cref{thm:apx-ratio-of-our-algo} shows the formal guarantees of \OurAlg{}.

\begin{algorithm}[htb]
\caption{\OurAlg{}}
\label{alg:ouralg}
\begin{algorithmic}[1]
\Statex \textbf{Input}: An $\OurProblem{}(f, h, r, \bV, \bb, \bp)$ instance and a tie-breaking ordering $\sigma: [r] \to [r]$
\Statex \textbf{Output}: A $\sigma$-type feasible selection $\bS = \bS^{(1:h)}$
\State Define $\bS^{(1)} = \ldots = \bS^{(h)} = \emptyset$ \Comment{Initialize selections}
\For{$t = 1, \ldots, h$}
    \State For $q \in [r]$, define $\#(q, \bs(b^{(1:0)}, b^{(1:0)})) = 0$ and
    \[
    x^{(t)}_q = \#(q, \bs(b^{(1:t)}, b^{(1:t)})) - \#(q, \bs(b^{(1:t-1)}, b^{(1:t-1)}))
    \]
    \State Let $\cM^{(t)} = (\bV^{(t)}, \cI^{(t)})$ be the partition matroid at
    \Statex\hspace{\algorithmicindent}time $t$, where $\cI^{(t)} = \{ \bS \subseteq \bV^{(t)}: | \bS \cap \bT_q | \leq x^{(t)}_q \}$.
    \For{$b^{(t)}$ times}
        \State Let $\bS = \bigcup_{\tau=1}^t \bS^{(\tau)}$ be the selection so far
        \State Add $e^* = \argmax\limits_{\substack{e \in \bV^{(t)} \setminus \bS\\ \bS^{(t)} \cup \{e\} \in \cI^{(t)}}} g(\bS, \{e\})$ to $\bS^{(t)}$
    \EndFor
\EndFor
\State Output $\bS^{(1)} \uplus \ldots \uplus \bS^{(h)}$
\end{algorithmic}
\end{algorithm}

\begin{restatable}{theorem}{apxratioofouralgo}
\label{thm:apx-ratio-of-our-algo}
Fix an arbitrary tie-breaking ordering $\sigma: [r] \to [r]$.
Assuming the objective function $f$ can be evaluated in constant time given any subset, \OurAlg{} runs in $\cO(\sum_{t=1}^h b^{(t)} \cdot (r + |\bV^{(t)}|))$ time and outputs $\bS = \uplus_{t=1}^h \bS^{(t)}$ such that
\begin{enumerate}
    \item For all $t \in [h]$, $\bS^{(t)} \subseteq \bV^{(t)}$ and $|\bS^{(t)}| = b^{(t)}$.
    \item The output $\bS$ is a $\sigma$-type feasible selection.
    \item For all $t \in [h]$, we have
    \[
    f(\bS^{(1)} \uplus \ldots \uplus \bS^{(t)}) \geq \frac{1}{2} \cdot f(\OPT_\sigma(b^{(1)}, \ldots, b^{(t)}))
    \]
    where $\OPT_\sigma(b^{(1)}, \ldots, b^{(t)})$ is an optimum $\sigma$-type feasible selection restricted to budgets $b^{(1)}, \ldots, b^{(t)}$.
\end{enumerate}
\end{restatable}

\subsection{Our theoretical bounds are tight in general}
\label{sec:impossibility}

\cref{prop:type-feasible-for-futureproof} shows that \emph{any} solution to \OurProblem{} violating \cref{def:type-feasible-selection} is brittle against future budget uncertainties while \cref{prop:impossibility-tight} show that the approximation ratio terms for \OurAlg{} are essentially tight.

\begin{restatable}{proposition}{typefeasibleforfutureproof}
\label{prop:type-feasible-for-futureproof}
Let \ALG{} be any deterministic algorithm to \OurProblem{} aiming to produce selections $\bS^{(1)}, \ldots, \bS^{(h)}$ that maximizes $\beta(b^{(1:t)})$ at each time step $t \in [h]$.
If \ALG{} does not minimize $|\bQ(\bS^{(t)})|$ for each $t \in [h]$, then there exists instances where it fails to minimize $\beta(b^{(1:t)})$ at each $t \in [h]$.
\end{restatable}

\begin{restatable}{proposition}{impossibilitytight}
\label{prop:impossibility-tight}
There exists instances where \OurAlg{} achieves an approximation ratio of $\leq \frac{1}{2}$ due to budget uncertainties and $\leq \frac{k}{k+1}$ due to $\sigma$-ordering.
\end{restatable}
\section{Experiments}
\label{sec:experiments}

We empirically evaluate the performance of the individual components of our proposed \OurTool{} tool.
See \cref{sec:appendix-experiment} for additional experimental details and access to source code and scripts.

\subsection{Experimental setup}

\subsubsection{Data.}

To model facility coverage, we use point-to-point walking distance estimates between 1km-by-1km grid cells across Ethiopia, based on the global friction surface from \cite{weiss2020global}, and population forecasts from WorldPop projections for the years 2026–2030 \cite{worldpop_eth_2015_2030}.\footnote{Population forecasts are inherently uncertain and the true objective function $f$ is not directly accessible.
In \cref{sec:appendix-proxy-function}, we show that the impact on facility coverage is not substantial.}
We are also given a set of existing facilities $\bS_{\text{existing}}$ and available budgets were chosen based on expert consultation.

\subsubsection{Regions analyzed.}

We focus on rural regions where access remains limited, i.e., major cities such as Addis Ababa are excluded. 
Specifically, we evaluate our algorithms on the regions of Afar, Benishangul Gumuz, and Somali in the main paper, and on the Sidama region in \cref{sec:appendix-experiment-sidama}.
This way we reflect Ethiopia's two dominant rural livelihood types.\footnote{In the main paper, we focus on less populated regions to better highlight the spatial coverage gains achieved by our \OurTool{}.}

There are around 30 districts in Afar and Sidama, around 20 in Benishangul Gumuz, and around 70 districts in Somali.
A larger budget was considered for Somali as it is much larger than the other regions.

\subsubsection{Distributional Policies (DP).}
\begin{figure}[htb]
    \centering
    \includegraphics[width=1.0\linewidth]{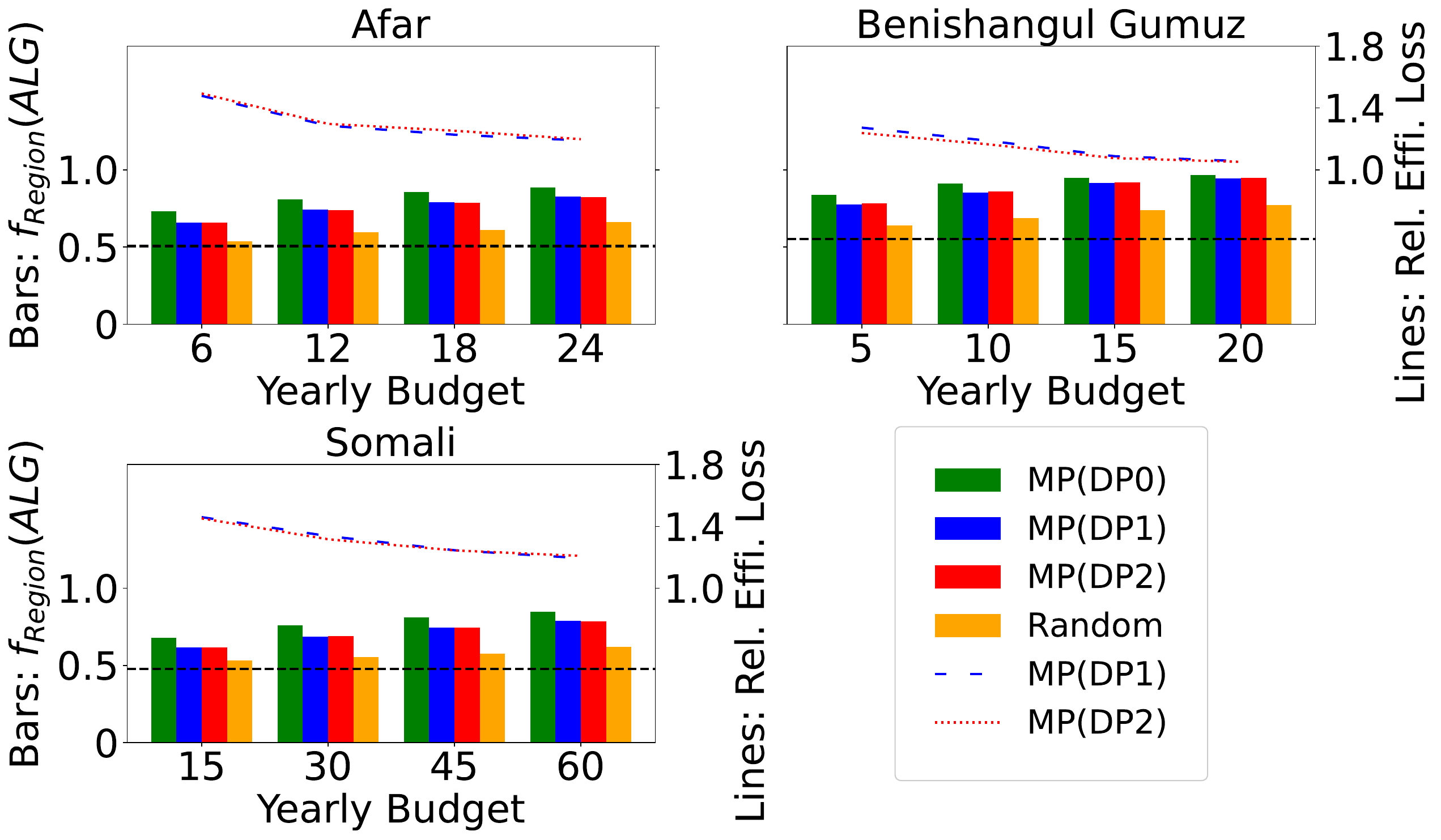}
    \caption{Coverage under varying annual budgets across the three regions of Afar, Benishangul Gumuz, and Somali. Bars show total population coverage by each policy, with the black line indicating existing coverage before any yearly budget was spent to build additional health posts. The dashed lines above indicate relative efficiency loss from enforcing proportional constraints for DP$1$ and DP$2$, compared to the unconstrained baseline (DP$0$).}
    \label{fig:budget}
\end{figure}
To ensure equitable allocation, each region may define a distributional policy by specifying target proportions $p_1, \ldots, p_r$ across $r$ districts.
These type-level constraints shape the facility selection process by prioritizing areas with greater unmet needs.
In our experiments, we evaluate three policy scenarios:
\begin{description}
    \item[DP0] No distributional constraint (i.e., $p_1 = \ldots = p_r = 0$), corresponding to unconstrained greedy selection.
    \item[DP1] Proportions $p_q$ are set based on unassisted home birth rates, favoring districts with poor maternal care access.
    \item[DP2] Proportions $p_q$ are set by the inverse of postnatal care coverage, favoring districts lacking early childhood services.
    \item[DP3] Proportions $p_q$ are set by the inverse of the average between home-birth rates and postnatal-care coverage, thereby balancing the two scenarios.
\end{description}
  
Both DP1 and DP2 target existing healthcare service gaps that can be provided by comprehensive health posts.
While our framework supports multiple active constraints simultaneously, we focus on one distributional policy at a time to preserve theoretical guarantees and interpretability.
In our experiments below, we write $\textsc{MP}(\text{DP}i)$, for $i \in \{0, 1, 2\}$, to denote the output of \cref{alg:ouralg} applied with the proportional constraints imposed by each policy.

\subsection{Experiment 1: Impact of budget on coverage}

In \cref{fig:budget}, the bars show how total coverage improves under different allocation strategies and budget levels over a five-year planning horizon.
As all regions are large relative to the available budget, coverage increases approximately linearly with budget in all cases.

For $i \in \{1,2\}$ in Afar, Benishangul Gumuz, and Somali, the dashed and dotted lines in \cref{fig:budget} show the ratios
\[
\frac{f_{\text{region}}(\textsc{MP}(\text{DP}0) \;\cup\; \bS_{\text{existing}}) - f_{\text{region}}(\bS_{\text{existing}})}{f_{\text{region}}(\textsc{MP}(\text{DP}i) \;\cup\; \bS_{\text{existing}}) - f_{\text{region}}(\bS_{\text{existing}})} \;,
\]
quantifying the trade-off incurred by distributional constraints, where $f$ is the coverage function defined in \cref{eq:ethiopia-objective-f}.

Note that we subtracted away the coverage from existing health posts $\bS_{\text{existing}}$ to measure the relative gain in coverage as existing health posts are assumed to remain operational throughout.
This ratio is similar to the notion of price of fairness \cite{caragiannis2012efficiency}: a ratio close to 1 indicates that enforcing proportional fairness incurs minimal efficiency loss, while higher values reflect stronger trade-offs.
As expected, these ratios improve with increasing budgets, suggesting that the proportional constraints come at a lower cost when more resources are available. 
Note that as the results of \textsc{MP}(DP$1$) and \textsc{MP}(DP$2$) were very similar, we did not report the performance of their average, \textsc{MP}(DP$3$).

\subsection{Experiment 2: Evaluating district-level equity}
\begin{figure}[htb]
    \centering
    \includegraphics[width=\columnwidth]{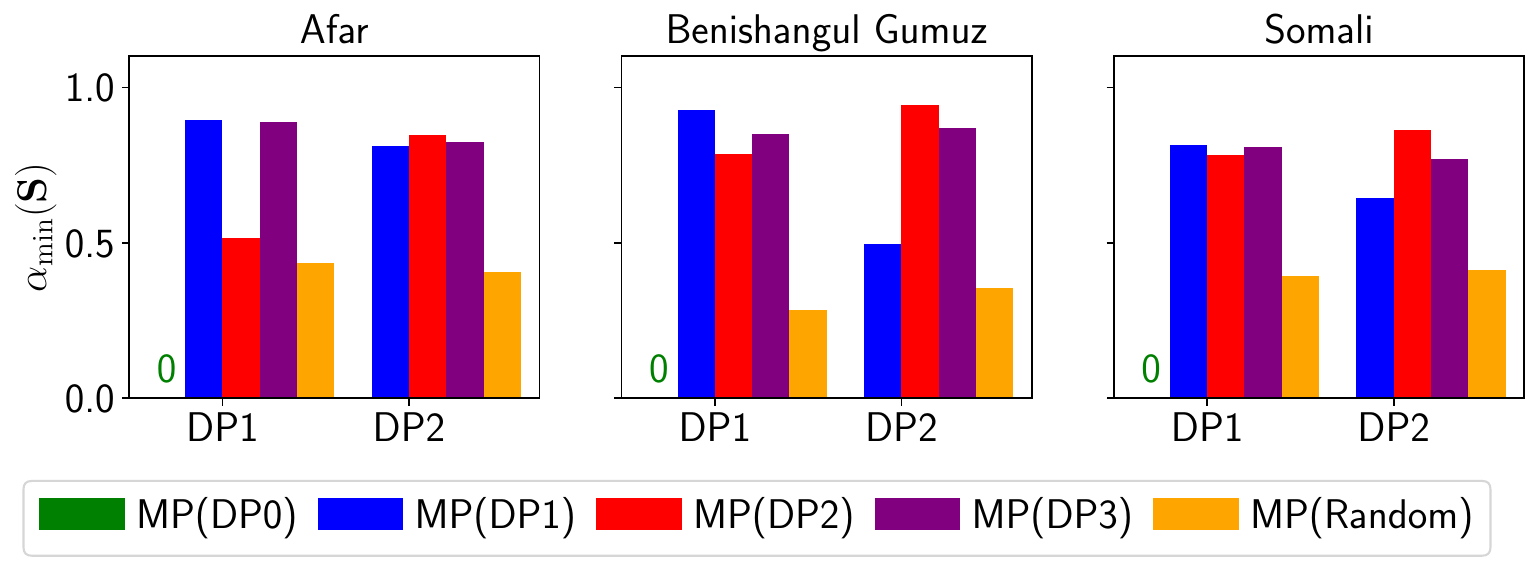}
    \caption{Minimum satisfaction ratio $\alpha_{\min}$ attained under each policy, with proportions defined by DP$1$ or DP$2$. Over five-year horizon, annual budgets of $12$ facilities were allocated for Afar, $10$ for Benishangul Gumuz, and $30$ for Somali.}
\label{fig:types}
\end{figure}
In Experiment 1, we demonstrated that imposing proportional constraints does not significantly degrade coverage. We now show that these constraints improve adherence to the intended proportions.
To this end, we examine the minimum satisfaction ratio, $\alpha_{\min}(\bS)$, achieved under each policy.
As shown in \cref{fig:types}, \textsc{MP}(DP$1$) and \textsc{MP}(DP$2$) outperform all other methods in the DP$1$ and DP$2$ metrics respectively, while \textsc{MP}(DP$3$) is somewhere in between and \textsc{MP}(DP$0$) consistently attains $\alpha_{\min} = 0$ .
Interestingly, in  Afar, \textsc{MP}(DP$1$) yields significant gain under DP$1$ than \textsc{MP}(DP$2$) does under DP$2$, whereas the reverse trend is observed in the Benishangul Gumuz and Somali regions.

\subsection{Experiment 3: High quality advice can help}
\label{sec:experiment-learning-augmented}

To evaluate the practical utility of our learning-augmented approach, we conduct a \emph{retrospective analysis} using the set of comprehensive health posts $\bA$ that were previously constructed by regional planners: we simulate the original decision-making scenario in which $|\bA|$ health posts must be selected from scratch.
We compare three strategies:\\
\textbf{1. Expert selection:} Use expert-curated selection $\bA$.\\
\textbf{2. Greedy baseline:} Apply \cref{alg:ouralg} to get $\bG$.\\
\textbf{3. Learning-augmented:} Use \cref{alg:LA-one-type} to produce $\bU$.

Although planning decisions are made at the regional level, our analysis is conducted at the district level to capture localized effects.
Across all districts in the Afar, Benishangul Gumuz, and Somali regions, the greedy baseline $\bG$ consistently outperforms the expert selection $\bA$.
Interestingly, we found $25$ districts within these regions in which our learning-augmented selection $\bU$ outperforms both, i.e., $f(\bU) > f(\bG) > f(\bA)$.
These results were obtained by selecting the best performing $\bU$ from ten runs of \cref{alg:LA-one-type}, each using a different random permutation of the original advice selection $\bA$.

\cref{fig:case} presents a representative example, showing a zoomed-in subarea of a district where all three strategies select two health posts (the rest of the district is relatively sparse), with $\frac{f(\bU) - f(\bG)}{f(\bG)} \approx 2\%$ and $\frac{f(\bU) - f(\bA)}{f(\bA)} \approx 20\%$.
The greedy algorithm, having made its first choice (the left location in $\bG$) based solely on immediate marginal gain, is subsequently limited in its second selection.
In contrast, \cref{alg:LA-one-type} could leverage on the leftward choice in $\bA$ to produce a more globally effective selection $\bU$.

\begin{figure}[htb]
    \centering
    \begin{subfigure}[b]{0.24\linewidth}
        \centering
        \includegraphics[width=\linewidth]{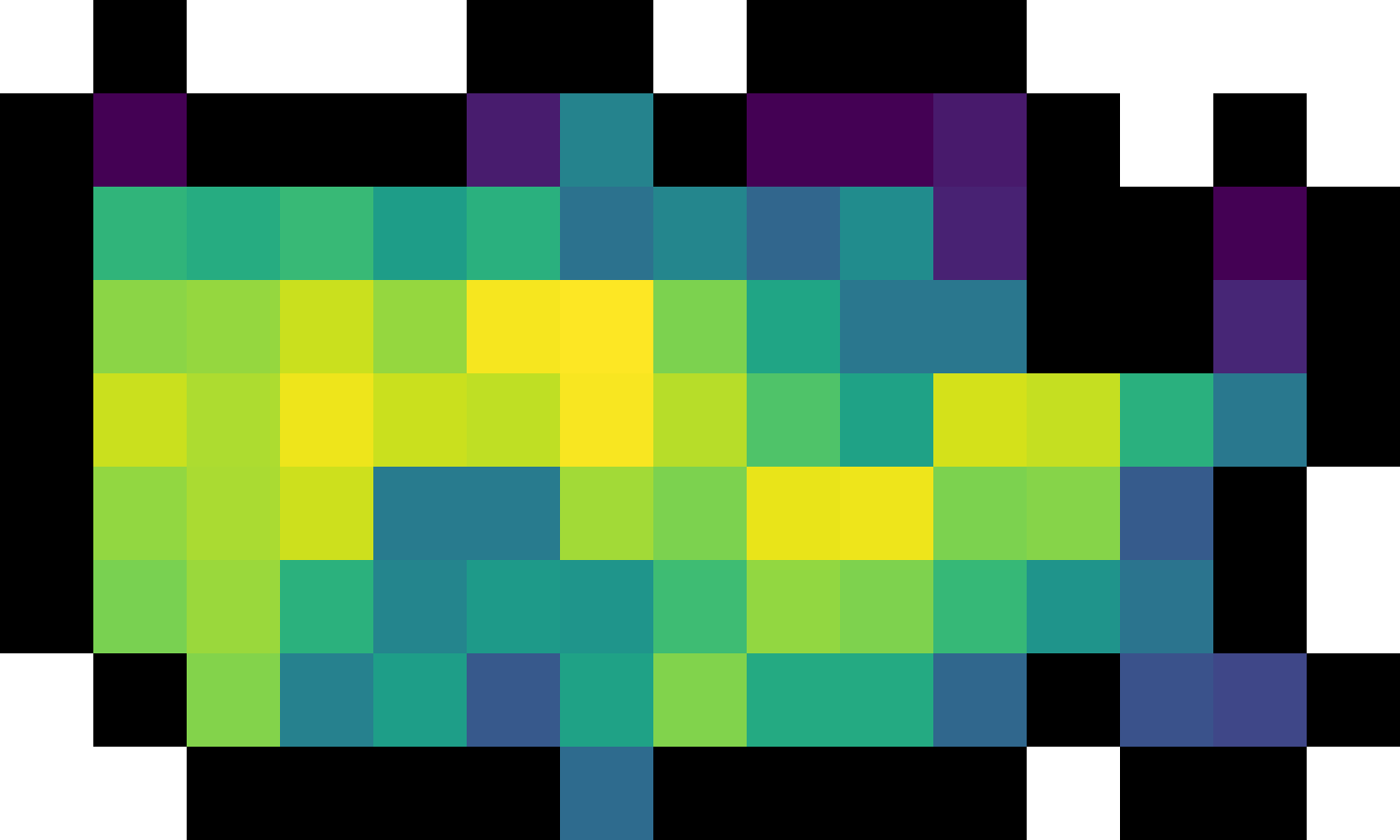}
        \caption{Population}
    \end{subfigure}
    \hfill
    \begin{subfigure}[b]{0.24\linewidth}
        \centering
        \includegraphics[width=\linewidth]{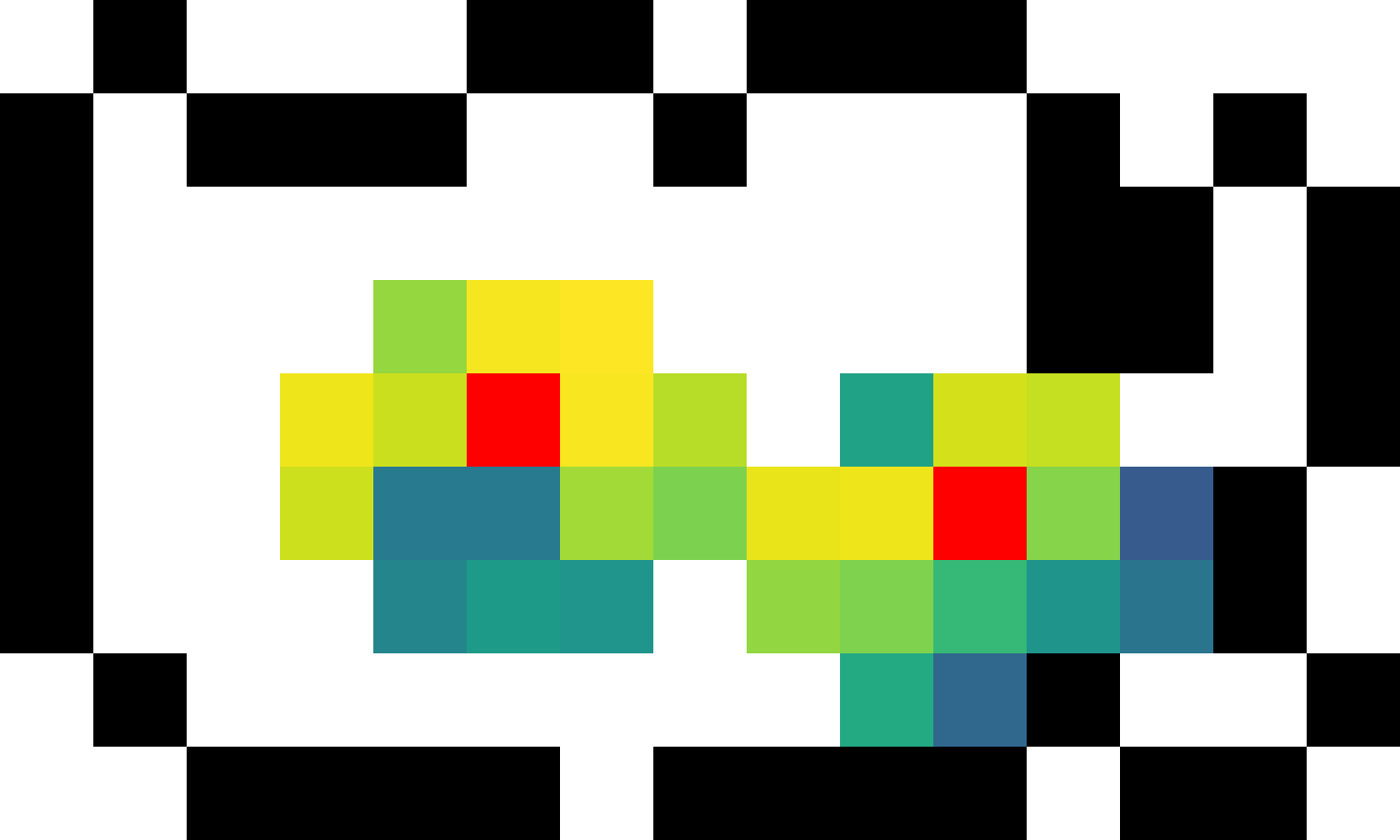}
        \caption{$\bG$}
    \end{subfigure}
    \hfill
    \begin{subfigure}[b]{0.24\linewidth}
        \centering
        \includegraphics[width=\linewidth]{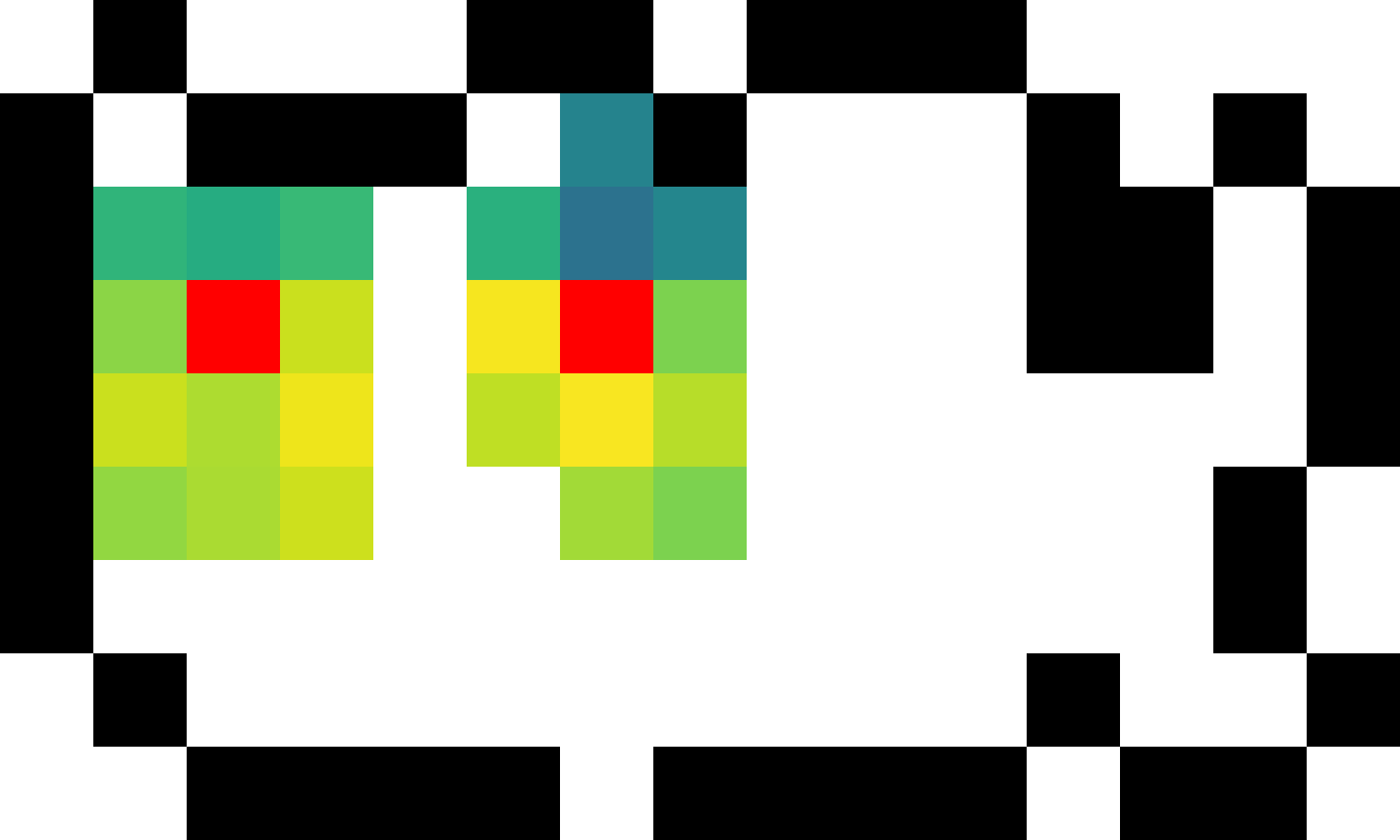}
        \caption{$\bA$}
    \end{subfigure}
    \hfill
    \begin{subfigure}[b]{0.24\linewidth}
        \centering
        \includegraphics[width=\linewidth]{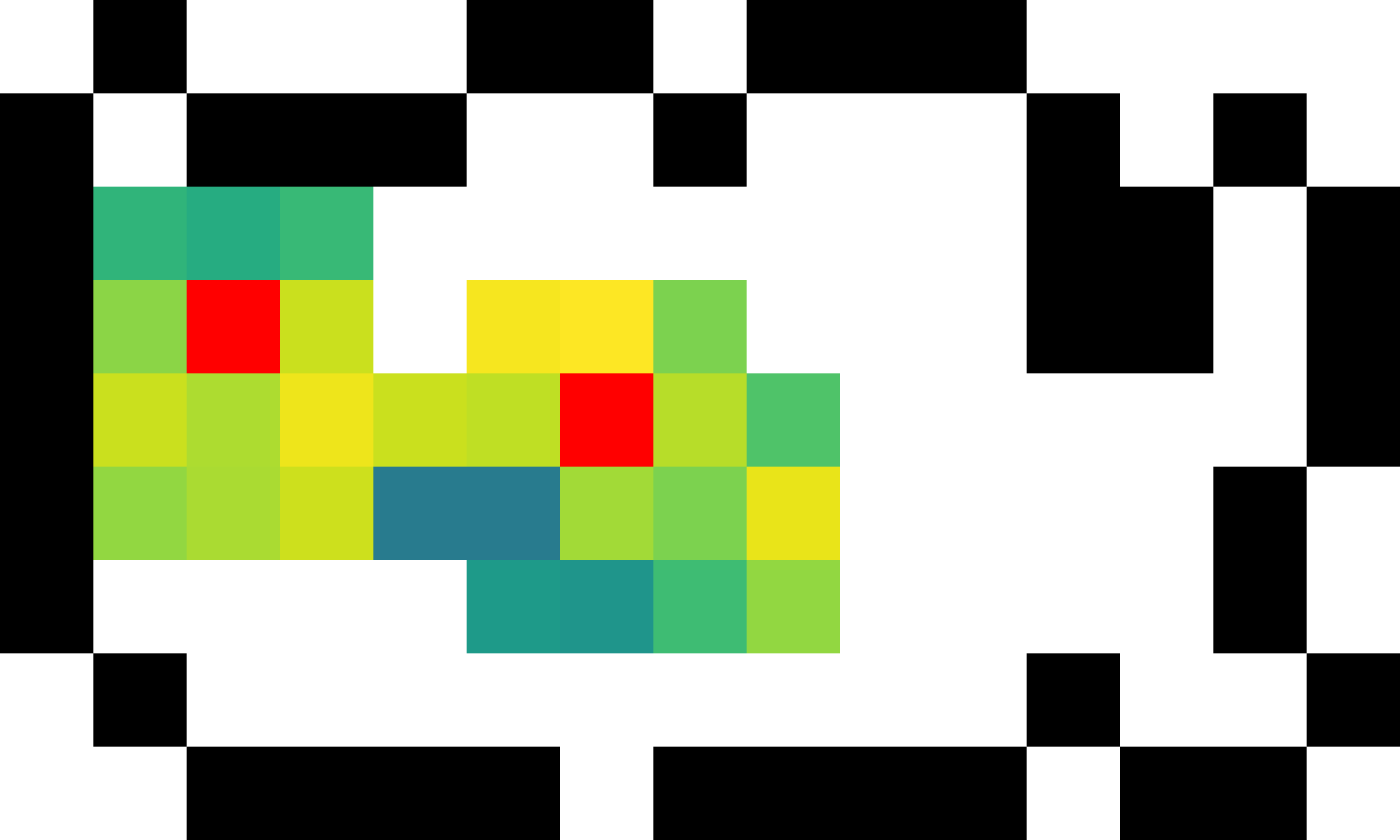}
        \caption{$\bU$}
    \end{subfigure}
    \caption{
    (a) 1km-by-1km population grid (log-scale): yellow indicates high population density, darker shades indicate lower density.
    (b,c,d) Red dots indicate the selected health post locations and the colored cells indicate the areas covered by these health posts (on the same log-scale rendering).
    }
    \label{fig:case}
\end{figure}

\section{Conclusion}
\label{sec:conclusion}

In partnership with Ethiopia's Ministry of Health and Public Health Institute, we developed a decision-support tool (\OurTool{}) to guide health facility planning.
\OurTool{} allows regional planners to solve their own instance of \OurProblem{} within a unified, principled framework that models the planning task as a constrained submodular optimization with online budget and global proportionality constraints.
On the technical front, we are the first to integrate learning-augmented algorithms with multi-step submodular optimization, enabling principled refinement of expert plans with worst-case guarantees.
Using real-world health and demographic data from Ethiopia, we show that our methods are effective in supporting equitable, data-driven infrastructure decisions across diverse regional contexts.

\section*{Acknowledgements}
The findings and conclusions in this report are those of the authors and do not represent any recommendations or official positions.

\bibliography{refs}

@article{nemhauser1978analysis,
  title={{An analysis of approximations for maximizing submodular set functions -- I}},
  author={Nemhauser, G. L. and Wolsey, L. A. and Fisher, M. L.},
  journal={Mathematical Programming},
  volume={14},
  pages={265--294},
  year={1978},
  publisher={Springer}
}

@book{fisher1978analysis,
  title={{An analysis of approximations for maximizing submodular set functions -- II}},
  author={Fisher, M. L. and Nemhauser, G. L. and Wolsey, L. A.},
  journal={Mathematical Programming Studies},
  volume={8},
  pages={73--87},
  year={1978},
  publisher={Springer}
}

@article{tanser2006modelling,
  title={{Modelling and understanding primary health care accessibility and utilization in rural South Africa: An exploration using a geographical information system}},
  author={Tanser, Frank and Gijsbertsen, Brice and Herbst, Kobus},
  journal={Social Science \& Medicine},
  volume={63},
  number={3},
  pages={691--705},
  year={2006},
  publisher={Elsevier}
}

@article{noor2003defining,
  title={{Defining equity in physical access to clinical services using geographical information systems as part of malaria planning and monitoring in Kenya}},
  author={Noor, A. M. and Zurovac, D. and Hay, S. I. and Ochola, S. A. and Snow, R. W.},
  journal={Tropical Medicine and International Health},
  volume={8},
  number={10},
  pages={917--926},
  year={2003},
  publisher={Wiley Online Library}
}

@article{cheng2021robust,
  title={{Robust facility location under demand uncertainty and facility disruptions}},
  author={Cheng, Chun and Adulyasak, Yossiri and Rousseau, Louis-Martin},
  journal={Omega},
  volume={103},
  pages={102429},
  year={2021},
  publisher={Elsevier}
}

@article{atamturk2007two,
  title={{Two-Stage Robust Network Flow and Design Under Demand Uncertainty}},
  author={Atamt{\"u}rk, Alper and Zhang, Muhong},
  journal={Operations Research},
  volume={55},
  number={4},
  pages={662--673},
  year={2007},
  publisher={INFORMS}
}

@article{baron2011facility,
  title={{Facility Location: A Robust Optimization Approach}},
  author={Baron, Opher and Milner, Joseph and Naseraldin, Hussein},
  journal={Production and Operations Management},
  volume={20},
  number={5},
  pages={772--785},
  year={2011},
  publisher={Sage}
}

@article{calinescu2011maximizing,
  title={{Maximizing a monotone submodular function subject to a matroid constraint}},
  author={Calinescu, Gruia and Chekuri, Chandra and Pal, Martin and Vondr{\'a}k, Jan},
  journal={SIAM Journal on Computing},
  volume={40},
  number={6},
  pages={1740--1766},
  year={2011},
  publisher={SIAM}
}

@inproceedings{dehghankar2025fair,
  title={{Fair Set Cover}},
  author={Dehghankar, Mohsen and Raychaudhury, Rahul and Sintos, Stavros and Asudeh, Abolfazl},
  booktitle={Conference on Knowledge Discovery and Data Mining (KDD)},
  pages={189--200},
  year={2025}
}

@article{lee2010maximizing,
  title={{Maximizing Nonmonotone Submodular Functions under Matroid or Knapsack Constraints}},
  author={Lee, Jon and Mirrokni, Vahab S. and Nagarajan, Viswanath and Sviridenko, Maxim},
  journal={SIAM Journal on Discrete Mathematics},
  volume={23},
  number={4},
  pages={2053--2078},
  year={2010},
  publisher={SIAM}
}

@article{ghaderi2013modeling,
  title={{Modeling the budget-constrained dynamic uncapacitated facility location–network design problem and solving it via two efficient heuristics: A case study of health care}},
  author={Ghaderi, Abdolsalam and Jabalameli, Mohammad Saeed},
  journal={Mathematical and Computer Modelling},
  volume={57},
  number={3-4},
  pages={382--400},
  year={2013},
  publisher={Elsevier}
}

@article{arabani2012facility,
  title={{Facility location dynamics: An overview of classifications and applications}},
  author={Arabani, Alireza Boloori and Farahani, Reza Zanjirani},
  journal={Computers \& Industrial Engineering},
  volume={62},
  number={1},
  pages={408--420},
  year={2012},
  publisher={Elsevier}
}

@article{asudeh2023maximizing,
  title={{Maximizing Coverage While Ensuring Fairness: A Tale of Conflicting Objectives}},
  author={Asudeh, Abolfazl and Berger-Wolf, Tanya and DasGupta, Bhaskar and Sidiropoulos, Anastasios},
  journal={Algorithmica},
  volume={85},
  pages={1287--1331},
  year={2023},
  publisher={Springer}
}

@article{gupta2022l_p,
  title={{Which $L_p$ norm is the Fairest? Approximations for Fair Facility Location Across All ``$p$''}},
  author={Gupta, Swati and Moondra, Jai and Singh, Mohit},
  journal={arXiv preprint arXiv:2211.14873},
  year={2022}
}

@inproceedings{celis2019classification,
  title={{Classification with Fairness Constraints: A Meta-Algorithm with Provable Guarantees}},
  author={Celis, L. Elisa and Huang, Lingxiao and Keswani, Vijay and Vishnoi, Nisheeth K.},
  booktitle={Conference on Fairness, Accountability, and Transparency (FAT)},
  pages={319--328},
  year={2019}
}

@article{tsang2019group,
  title={{Group-Fairness in Influence Maximization}},
  author={Tsang, Alan and Wilder, Bryan and Rice, Eric and Tambe, Milind and Zick, Yair},
  journal={International Joint Conference on Artificial Intelligence (IJCAI)},
  year={2019}
}

@inproceedings{babaioff2007matroids,
  title={{Matroids, Secretary Problems, and Online Mechanisms}},
  author={Babaioff, Moshe and Immorlica, Nicole and Kleinberg, Robert},
  booktitle={Symposium on Discrete Algorithms (SODA)},
  pages={434--443},
  year={2007}
}

@inproceedings{chekuri2011submodular,
  title={{Submodular Function Maximization via the Multilinear Relaxation and Contention Resolution Schemes}},
  author={Chekuri, Chandra and Vondr{\'a}k, Jan and Zenklusen, Rico},
  booktitle={ACM Symposium on Theory of Computing (STOC)},
  pages={783--792},
  year={2011}
}

@article{bateni2013submodular,
  title={{Submodular Secretary Problem and Extensions}},
  author={Bateni, MohammadHossein and Hajiaghayi, MohammadTaghi and Zadimoghaddam, Morteza},
  journal={ACM Transactions on Algorithms (TALG)},
  volume={9},
  number={4},
  pages={1--23},
  year={2013},
  publisher={Association for Computing Machinery (ACM)}
}

@article{santiago2025constant,
  title={{Constant-competitiveness for random assignment Matroid secretary without knowing the Matroid}},
  author={Santiago, Richard and Sergeev, Ivan and Zenklusen, Rico},
  journal={Mathematical Programming},
  volume={210},
  number={1},
  pages={815--846},
  year={2025},
  publisher={Springer}
}

@article{cristi2024online,
  title={{Online Matroid Embeddings}},
  author={Cristi, Andr{\'e}s and D{\"u}tting, Paul and Kleinberg, Robert and Leme, Renato Paes},
  journal={arXiv preprint arXiv:2407.10316},
  year={2024}
}

@article{recski1973partitional,
  title={{On partitional matroids with applications}},
  author={Recski, Andras},
  journal={Coll. Math. Soc. J. Bolyai},
  volume={10},
  pages={1169--1179},
  year={1973}
}

@article{weiss2020global,
  title={{Global maps of travel time to healthcare facilities}},
  author={Weiss, D. J. and Nelson, A. and Vargas-Ruiz, C. A. and Gligori{\'c}, K. and Bavadekar, S. and Gabrilovich, E. and Bertozzi-Villa, A. and Rozier, J. and Gibson, H. S. and Shekel, T. and et al.},
  journal={Nature Medicine},
  volume={26},
  pages={1835--1838},
  year={2020},
  publisher={Nature Publishing Group}
}

@misc{worldpop_eth_2015_2030,
  title        = {WorldPop Population Counts 2015--2030: Ethiopia — Spatial Distribution of Population},
  author       = {{WorldPop}},
  howpublished = {\url{https://data.humdata.org/dataset/worldpop-population-counts-2015-2030-eth}},
  year         = {2025},
  month        = jan,
  note         = {Constrained estimates of total population per 100m grid cell in Ethiopia for 2015–2030, Geotiff format, CC BY4.0 license},
  institution  = {WorldPop, School of Geography \& Environmental Science, University of Southampton / Humanitarian Data Exchange},
}

@article{goundan2007revisiting,
  title={{Revisiting the Greedy Approach to Submodular Set Function Maximization}},
  author={Goundan, Pranava R. and Schulz, Andreas S.},
  journal={Optimization Online},
  year={2007}
}

@article{lykouris2021competitive,
  title={{Competitive Caching with Machine Learned Advice}},
  author={Lykouris, Thodoris and Vassilvitskii, Sergei},
  journal={Journal of the ACM (JACM)},
  volume={68},
  number={4},
  pages={1--25},
  year={2021},
  publisher={ACM New York, NY}
}

@inproceedings{gollapudi2019online,
  title={{Online Algorithms for Rent-or-Buy with Expert Advice}},
  author={Gollapudi, Sreenivas and Panigrahi, Debmalya},
  booktitle={International Conference on Machine Learning (ICML)},
  pages={2319--2327},
  year={2019},
}

@article{wang2020online,
  title={{Online Algorithms for Multi-shop Ski Rental with Machine Learned Advice}},
  author={Wang, Shufan and Li, Jian and Wang, Shiqiang},
  journal={Neural Information Processing Systems (NeurIPS)},
  volume={33},
  pages={8150--8160},
  year={2020}
}

@inproceedings{angelopoulos2020online,
  title={{Online Computation with Untrusted Advice}},
  author={Angelopoulos, Spyros and D{\"u}rr, Christoph and Jin, Shendan and Kamali, Shahin and Renault, Marc},
  booktitle={Innovations in Theoretical Computer Science Conference (ITCS)},
  year={2020},
  organization={Schloss Dagstuhl-Leibniz-Zentrum f{\"u}r Informatik}
}

@article{purohit2018improving,
  title={{Improving Online Algorithms via ML Predictions}},
  author={Purohit, Manish and Svitkina, Zoya and Kumar, Ravi},
  journal={Neural Information Processing Systems (NeurIPS)},
  volume={31},
  year={2018}
}

@inproceedings{lattanzi2020online,
  title={{Online Scheduling via Learned Weights}},
  author={Lattanzi, Silvio and Lavastida, Thomas and Moseley, Benjamin and Vassilvitskii, Sergei},
  booktitle={ACM-SIAM Symposium on Discrete Algorithms (SODA)},
  pages={1859--1877},
  year={2020},
  organization={SIAM}
}

@article{bamas2020learning,
  title={{Learning Augmented Energy Minimization via Speed Scaling}},
  author={Bamas, {\'E}tienne and Maggiori, Andreas and Rohwedder, Lars and Svensson, Ola},
  journal={Advances in Neural Information Processing Systems (NeurIPS)},
  volume={33},
  pages={15350--15359},
  year={2020}
}

@inproceedings{antoniadis2022novel,
  title={{A Novel Prediction Setup for Online Speed-Scaling}},
  author={Antoniadis, Antonios and Jabbarzade, Peyman and Shahkarami, Golnoosh},
  booktitle={Scandinavian Symposium and Workshops on Algorithm Theory (SWAT)},
  year={2022},
  organization={Schloss Dagstuhl-Leibniz-Zentrum f{\"u}r Informatik}
}

@inproceedings{kraska2018case,
  title={{The Case for Learned Index Structures}},
  author={Kraska, Tim and Beutel, Alex and Chi, Ed H. and Dean, Jeffrey and Polyzotis, Neoklis},
  booktitle={International Conference on Management of Data (SIGMOD)},
  pages={489--504},
  year={2018}
}

@inproceedings{mitzenmacher2018model,
  title={{A Model for Learned Bloom Filters, and Optimizing by Sandwiching}},
  author={Mitzenmacher, Michael},
  booktitle={Neural Information Processing Systems (NeurIPS)},
  year={2018}
}

@article{antoniadis2020secretary,
  title={{Secretary and Online Matching Problems with Machine Learned Advice}},
  author={Antoniadis, Antonios and Gouleakis, Themis and Kleer, Pieter and Kolev, Pavel},
  journal={Neural Information Processing Systems (NeurIPS)},
  volume={33},
  pages={7933--7944},
  year={2020}
}

@inproceedings{dutting2021secretaries,
  title={{Secretaries with Advice}},
  author={D{\"u}tting, Paul and Lattanzi, Silvio and Paes Leme, Renato and Vassilvitskii, Sergei},
  booktitle={ACM Conference on Economics and Computation},
  pages={409--429},
  year={2021}
}

@inproceedings{choo2024online,
  title={{Online bipartite matching with imperfect advice}},
  author={Choo, Davin and Gouleakis, Themis and Ling, Chun Kai and Bhattacharyya, Arnab},
  booktitle={International Conference on Machine Learning (ICML)},
  pages = {8762--8781},
  year={2024},
}

@inproceedings{bernardini2022universal,
  title={{A Universal Error Measure for Input Predictions Applied to Online Graph Problems}},
  author={Bernardini, Giulia and Lindermayr, Alexander and Marchetti-Spaccamela, Alberto and Megow, Nicole and Stougie, Leen and Sweering, Michelle},
  booktitle={Neural Information Processing Systems (NeurIPS)},
  year={2022}
}

@inproceedings{gouleakis2023learning,
  title={{Learning-Augmented Algorithms for Online TSP on the Line}},
  author={Gouleakis, Themis and Lakis, Konstantinos and Shahkarami, Golnoosh},
  booktitle={AAAI Conference on Artificial Intelligence (AAAI)},
  pages={11989--11996},
  year={2023}
}

@inproceedings{dinitz2021faster,
  title={{Faster Matchings via Learned Duals}},
  author={Dinitz, Michael and Im, Sungjin and Lavastida, Thomas and Moseley, Benjamin and Vassilvitskii, Sergei},
  booktitle={Neural Information Processing Systems (NeurIPS)},
  pages={10393--10406},
  year={2021}
}

@inproceedings{chen2022faster,
  title={{Faster Fundamental Graph Algorithms via Learned Predictions}},
  author={Chen, Justin and Silwal, Sandeep and Vakilian, Ali and Zhang, Fred},
  booktitle={International Conference on Machine Learning (ICML)},
  pages={3583--3602},
  year={2022},
}

@inproceedings{choo2023active,
  title={{Active causal structure learning with advice}},
  author={Choo, Davin and Gouleakis, Themis and Bhattacharyya, Arnab},
  booktitle={International Conference on Machine Learning (ICML)},
  pages={5838--5867},
  year={2023}
}

@inproceedings{agrawal2022learning,
  title={Learning-augmented mechanism design: Leveraging predictions for facility location},
  author={Agrawal, Priyank and Balkanski, Eric and Gkatzelis, Vasilis and Ou, Tingting and Tan, Xizhi},
  booktitle={Proceedings of the 23rd ACM Conference on Economics and Computation},
  pages={497--528},
  year={2022}
}

@inproceedings{gkatzelis2022improved,
  title={{Improved Price of Anarchy via Predictions}},
  author={Gkatzelis, Vasilis and Kollias, Kostas and Sgouritsa, Alkmini and Tan, Xizhi},
  booktitle={ACM Conference on Economics and Computation (EC)},
  pages={529--557},
  year={2022}
}

@inproceedings{choo2025learning,
  title={Learning-Augmented Online Bipartite Fractional Matching},
  author={Choo, Davin and Jin, Billy and Shin, Yongho},
  booktitle={Conference on Neural Information Processing Systems (NeurIPS)},
  year={2025}
}

@inproceedings{bhattacharyya2025product,
  title={{Product distribution learning with imperfect advice}},
  author={Bhattacharyya, Arnab and Choo, Davin and George John, Philips and Gouleakis, Themis},
  booktitle={Conference on Neural Information Processing Systems (NeurIPS)},
  year={2025}
}

@inproceedings{bhattacharyya2025learning,
  title={{Learning multivariate Gaussians with imperfect advice}},
  author={Bhattacharyya, Arnab and Choo, Davin and George John, Philips and Gouleakis, Themis},
  booktitle={International Conference on Machine Learning (ICML)},
  year={2025}
}

@article{mitzenmacher2022algorithms,
  author={Mitzenmacher, Michael and Vassilvitskii, Sergei},
  title={{Algorithms with Predictions}},
  journal={Communications of the ACM},
  publisher={Association for Computing Machinery (ACM)},
  volume={65},
  number={7},
  pages={33--35},
  year={2022}
}

@inproceedings{agarwal2024learning,
  title={{Learning-Augmented Dynamic Submodular Maximization}},
  author={Agarwal, Arpit and Balkanski, Eric},
  booktitle={Neural Information Processing Systems (NeurIPS)},
  pages={14148--14176},
  year={2024}
}

@inproceedings{liu2024predicted,
  title={{The Predicted-Updates Dynamic Model: Offline, Incremental, and Decremental to Fully Dynamic Transformations}},
  author={Liu, Quanquan C. and Srinivas, Vaidehi},
  booktitle={Conference on Learning Theory (COLT)},
  pages={3582--3641},
  year={2024},
  organization={Proceedings of Machine Learning Research (PMLR)}
}

@misc{OCHA_Ethiopia_AdminBoundaries_2025,
  title={{Ethiopia - Subnational Administrative Boundaries}},
  author={United Nations OCHA},
  year={2023},
  note={Available at \url{https://data.humdata.org/dataset/cod-ab-eth}}
}

@article{caragiannis2012efficiency,
  title={{The Efficiency of Fair Division}},
  author={Caragiannis, Ioannis and Kaklamanis, Christos and Kanellopoulos, Panagiotis and Kyropoulou, Maria},
  journal={Theory of Computing Systems},
  volume={50},
  pages={589--610},
  year={2012},
  publisher={Springer}
}

@book{wang2016ethiopia,
  title={{Ethiopia Health Extension Program: An Institutionalized Community Approach for Universal Health Coverage}},
  author={Wang, Huihui and Tesfaye, Roman and Ramana, Gandham N. V. and Chekagn, Chala Tesfaye},
  year={2016},
  publisher={World Bank}
}

@misc{moh2020roadmap,
  title={{Realizing Universal Health Coverage through Primary Healthcare: A Roadmap for Optimizing the Ethiopian Health Extension Program (2020--2035)}},
  author={{Ministry of Health [Ethiopia]}},
  year={2020},
  publisher={Ministry of Health Addis Ababa, Ethiopia}
}

@article{hendrix2023estimated,
  title={{Estimated travel time and staffing constraints to accessing the Ethiopian health care system: A two-step floating catchment area analysis}},
  author={Hendrix, Nathaniel and Warkaye, Samson and Tesfaye, Latera and Woldekidan, Mesfin Agachew and Arja, Asrat and Sato, Ryoko and Memirie, Solomon Tessema and Mirkuzie, Alemnesh H and Getnet, Fentabil and Verguet, St{\'e}phane},
  journal={Journal of Global Health},
  volume={13},
  pages={04008},
  year={2023}
}

@article{getnet2025inequalities,
  title={{Inequalities in tuberculosis control in Ethiopia: A district-level distributional modelling analysis}},
  author={Getnet, Fentabil and Forzy, Tom and Tesfaye, Latera and Misganaw, Awoke and Memirie, Solomon Tessema and Geremew, Shewayiref and Berheto, Tezera Moshago and Wendrad, Naod and Yihun, Bantalem Yeshanew and Mirutse, Mizan Kiros and et al.},
  journal={Tropical Medicine \& International Health},
  volume={30},
  number={1},
  pages={31--42},
  year={2025},
  publisher={Wiley Online Library}
}

@misc{unfpa2022ethiopia,
  author       = {{UNFPA}},
  title        = {Ethiopia population, 2022},
  year         = {2022},
  note         = {Accessed via UNFPA resources},
}

@misc{worldometers2022,
  author       = {{Worldometers}},
  title        = {Ethiopia Population (2022) -- Worldometer},
  year         = {2022},
  howpublished = {\url{https://www.worldometers.info/world-population/ethiopia-population/}},
  note         = {Accessed July 2025},
}

@techreport{Teklu2020,
  author      = {Teklu, A. M. and Alemayehu, Y. K. and Medhin, G. and others},
  title       = {National Assessment of the Ethiopian Health Extension Program: Final Master Report},
  year        = {2020},
  institution = {MERQ Consultancy PLC},
  address     = {Addis Ababa, Ethiopia}
}

@inproceedings{ergun2022learning,
  title={{Learning-Augmented $k$-means Clustering}},
  author={Ergun, Jon C. and Feng, Zhili and Silwal, Sandeep and Woodruff, David and Zhou, Samson},
  booktitle={International Conference on Learning Representations (ICLR)},
  year={2022}
}

@inproceedings{braverman2024learning,
  title={{Learning-Augmented Maximum Independent Set}},
  author={Braverman, Vladimir and Dharangutte, Prathamesh and Shah, Vihan and Wang, Chen},
  booktitle={Approximation, Randomization, and Combinatorial Optimization. Algorithms and Techniques (APPROX/RANDOM)},
  year={2024}
}

@inproceedings{cohen2024learning,
  title={{Learning-Augmented Approximation Algorithms for Maximum Cut and Related Problems}},
  author={Cohen-Addad, Vincent and d'Orsi, Tommaso and Gupta, Anupam and Lee, Euiwoong and Panigrahi, Debmalya},
  booktitle={Advances in Neural Information Processing Systems (NeurIPS)},
  pages={25961--25980},
  year={2024}
}

@article{owen1998strategic,
  title={{Strategic facility location: A review}},
  author={Owen, Susan Hesse and Daskin, Mark S.},
  journal={European Journal of Operational Research},
  volume={111},
  number={3},
  pages={423--447},
  year={1998},
  publisher={Elsevier}
}

@article{snyder2006facility,
  title={{Facility location under uncertainty: a review}},
  author={Snyder, Lawrence V.},
  journal={IIE Transactions},
  volume={38},
  number={7},
  pages={547--564},
  year={2006},
  publisher={Taylor \& Francis}
}

@article{ray2008accessmod,
  title={{AccessMod 3.0: computing geographic coverage and accessibility to health care services using anisotropic movement of patients}},
  author={Ray, Nicolas and Ebener, Steeve},
  journal={International Journal of Health Geographics},
  volume={7},
  number={63},
  year={2008},
  publisher={Springer}
}

@article{hierink2023geospatial,
  title={{A geospatial analysis of accessibility and availability to implement the primary healthcare roadmap in Ethiopia}},
  author={Hierink, Fleur and Oladeji, Olusola and Robins, Ann and Mu{\~n}iz, Maria F. and Ayalew, Yejimmawerk and Ray, Nicolas},
  journal={Communications Medicine},
  volume={3},
  number={140},
  year={2023},
  publisher={Nature Publishing}
}

@article{shi2020artificial,
  title={Artificial intelligence for social good: A survey},
  author={Shi, Zheyuan Ryan and Wang, Claire and Fang, Fei},
  journal={arXiv preprint arXiv:2001.01818},
  year={2020}
}

@article{nair2022adviser,
  title={Adviser: Ai-driven vaccination intervention optimiser for increasing vaccine uptake in nigeria},
  author={Nair, Vineet and Prakash, Kritika and Wilbur, Michael and Taneja, Aparna and Namblard, Corinne and Adeyemo, Oyindamola and Dubey, Abhishek and Adereni, Abiodun and Tambe, Milind and Mukhopadhyay, Ayan},
  journal={arXiv preprint arXiv:2204.13663},
  year={2022}
}

@inproceedings{ou2022networked,
  title={Networked restless multi-armed bandits for mobile interventions},
  author={Ou, Han-Ching and Siebenbrunner, Christoph and Killian, Jackson and Brooks, Meredith B and Kempe, David and Vorobeychik, Yevgeniy and Tambe, Milind},
  booktitle={Conference on Autonomous Agents and Multiagent Systems (AAMAS)},
  year={2022}
}

@article{shariff2012location,
  title={Location allocation modeling for healthcare facility planning in Malaysia},
  author={Shariff, SS Radiah and Moin, Noor Hasnah and Omar, Mohd},
  journal={Computers \& Industrial Engineering},
  volume={62},
  number={4},
  pages={1000--1010},
  year={2012},
  publisher={Elsevier}
}
\bibliographystyle{alpha}

\newpage
\appendix
\section{Further related work}
\label{sec:appendix-related-work}

In this section, we provide additional related work that was not mentioned in \cref{sec:related-work} due to space constraints.

\subsection{Submodular maximization under constraints}

In \cref{sec:related-work}, we also mentioned a locally greedy algorithm of \cite{fisher1978analysis} which yields a tight $\frac{1}{k+1}$ approximation to the problem of maximizing a non-decreasing submodular function subject to $k$ matroid intersections.
We provide the pseudocode of this locally greedy algorithm in \cref{alg:localgreedy}.

\begin{algorithm}[htb]
\caption{\textsc{LocalGreedy}}
\label{alg:localgreedy}
\begin{algorithmic}[1]
\Statex \textbf{Input}: Objective function $f$, time horizon $h$, elements $\bV = \bV^{(1)} \uplus \ldots \uplus \bV^{(h)}$, $k$ matroid constraints $\cM_1 = (\bV, \cI_1), \ldots, \cM_k = (\bV, \cI_k)$ over $\bV$
\Statex \textbf{Output}: $\bS \subseteq \bV$ s.t. $\bS \cap \bV^{(t)} \in \cI_1 \cap \ldots \cap \cI_k, \forall t \in [h]$
\State Initialize $\bS = \emptyset$
\For{$t = 1, \ldots, h$}
    \While{$\exists e \in \bV^{(t)} \setminus \bS$ s.t. $\bS \cup \{e\} \in \cap_{i=1}^k \cI_i$}
        \State Define $e^* = \argmax\limits_{\bS \cup \{e\} \in \cap_{i=1}^k \cI_i} \bigg\{ f(\bS \cup \{e\}) - f(\bS) \bigg\}$
        \State Add $e^*$ to $\bS$
    \EndWhile
\EndFor
\State \Return $\bS$
\end{algorithmic}
\end{algorithm}

Maximizing monotone submodular functions under matroid constraints has a rich history in offline settings \cite{nemhauser1978analysis, calinescu2011maximizing, chekuri2011submodular}.
While work on submodular maximization under multiple matroid constraints such as \cite{lee2010maximizing} hints at possible modeling of our setting (one matroid for budgets and one for distributional constraints), these are offline formulations with static feasibility and do not capture our problem's adaptivity or temporal uncertainty.
In the online setting, one canonical problem that of the submodular secretary problem \cite{babaioff2007matroids, bateni2013submodular}, but it assumes that the matroid is fixed and known, while elements arrive sequentially in random or adversarial order.
In contrast, our setting inverts the uncertainty: the ground set is known in advance, but feasibility constraints (budgets) are revealed over time.
Moreover, secretary-style methods typically do not support global type-based constraints across time and assumes that arrivals are in random order.
More recent work such as those by \cite{celis2019classification} and \cite{tsang2019group} introduce global representation constraints into this framework in the form of fairness constraints.
The state-of-the-art approaches for solving the above-mentioned problems typically rely on full knowledge of the input, employ continuous relaxations, or local exchange techniques.
As such, one cannot directly apply their strategies in our settings where decisions are irrevocable and constraints evolve over time.

\subsection{Facility location under temporal and representation constraints}

Classical and robust facility location problems are extensively studied (see e.g., surveys by \cite{arabani2012facility,owen1998strategic}), including in dynamic and uncertain environments \cite{snyder2006facility,cheng2021robust, atamturk2007two,baron2011facility}.
However, most of these models aim to optimize a final objective after uncertainty is resolved.
Our approach departs from this paradigm by seeking robust decisions at each time step, enabling responsiveness to time-varying demand.
Healthcare-specific studies, such as \cite{gupta2022l_p}, introduce incremental facility placement under multiple normative criteria, including fairness.
However, their work also lacks support for temporally evolving feasibility constraints.
Additionally, they assume static facility needs and a metric distance space, which may not hold in contexts like rural Ethiopia.
Another work by \cite{ghaderi2013modeling} proposes a budgeted healthcare access model in Iran, but does not consider representation or proportionality constraints across population types. 

\subsection{Measure distance to facilities}

\cite{tanser2006modelling} proposed a model for measuring the actual distances to health facilities. \cite{noor2003defining} incorporated equity considerations into their model to enhance the planning and monitoring of health facilities.  
Using the AccessMod \cite{ray2008accessmod} tool for computing the coverage of health facilities, \cite{hierink2023geospatial} analyze the coverage in Ethiopia, and specifically in the Somali region. They suggest that the coverage in the Somali region is insufficient.
Recently, \cite{weiss2020global} proposed a method to quantify access to health facilities globally. Using 
However, these works emphasize the importance of locating the facilities appropriately, but do not provide algorithmic methods for doing so.
Our model explicitly incorporates time-sensitive population variation and provides algorithmic support for adaptive placement subject to budget and distributional considerations.

\subsection{Representation-constrained coverage and online constraints}

Representation constraints have also been studied in coverage problems.
For example, \cite{asudeh2023maximizing} consider a max-$k$-coverage problem with group fairness constraints but enforce rigid group-wise coverage parity, which can lead to infeasibility.
In contrast, our model allows flexible type-level bounds, maintaining feasibility and realism.
\cite{dehghankar2025fair} explore proportionality in set cover, applying constraints to the sets rather than the elements.
Their formulation targets minimum-cost full coverage, which diverges from our goal of budget-constrained partial coverage under temporal uncertainty.
More recently, \cite{cristi2024online} and \cite{santiago2025constant} explore online problems where constraints, rather than elements, arrive over time, similar to our evolving feasibility model.
However, these papers focus on maximizing simple linear (weight-based) objectives and do not directly extend to submodular objectives or type-based distributional constraints.

\subsection{AI For Social Impact}
Recently, many works are focused on developing AI tools for social impact (see e.g., \cite{shi2020artificial} for a broad survey). 
The aforementioned survey emphasizes that the healthcare domain has received the most attention among other application domains. 
In the context of vaccination, \cite{nair2022adviser} introduced a framework that optimizes the allocation of health interventions under uncertainty, with the aim of improving vaccination uptake in Nigeria.
Another work by \cite{ou2022networked}, proposed an algorithm for optimizing mobile interventions in India.
However, many of these works focus on dynamic/reactive intervention allocation problems while we focus on proactive, infrastructure-level investment planning. 
Another line of research examines health-facility location. \cite{shariff2012location, ghaderi2013modeling} develop models for optimizing health-care coverage in Malaysia and Iran, respectively. However, these studies do not incorporate representation or proportionality constraints across different population groups.

\subsection{Learning-augmented algorithms}

Since the seminal work of \cite{lykouris2021competitive}, there has been growing interest in designing algorithms that incorporate unreliable advice, with performance guarantees that degrade gracefully with advice quality.
This framework has been especially influential in online optimization, where advice can serve as a proxy for unknown future inputs.
Notable applications include ski-rental \cite{gollapudi2019online, wang2020online, angelopoulos2020online}, non-clairvoyant scheduling \cite{purohit2018improving}, job scheduling \cite{lattanzi2020online, bamas2020learning, antoniadis2022novel}, data structures with predictions (e.g., learned indexes \cite{kraska2018case}, Bloom filters \cite{mitzenmacher2018model}), online selection and matching \cite{antoniadis2020secretary, dutting2021secretaries, choo2024online, choo2025learning}, and online TSP \cite{bernardini2022universal, gouleakis2023learning}.
More recent work has expanded this paradigm to graph algorithms \cite{dinitz2021faster, chen2022faster}, causal discovery \cite{choo2023active}, mechanism design \cite{agrawal2022learning, gkatzelis2022improved}, improving approximation guarantees in offline NP-hard problems \cite{ergun2022learning, braverman2024learning, cohen2024learning}, and distribution learning \cite{bhattacharyya2025learning, bhattacharyya2025product}.
For a broader overview, we refer the reader to the survey by \cite{mitzenmacher2022algorithms} and the website \url{https://algorithms-with-predictions.github.io/} which tracks papers on this research topic.

\section{Deferred technical details for \texorpdfstring{\cref{sec:formulation}}{Section 2}}

\subsection{Showing that objective function \texorpdfstring{\cref{eq:ethiopia-objective-f}}{Eq. (1)} is non-decreasing and submodular}
\label{sec:appendix-ethiopia-objective-f}

\begin{proposition}
\label{prop:ethiopia-f-is-monotone-submodular}
The objective function defined in \cref{eq:ethiopia-objective-f} is non-decreasing and submodular.
\end{proposition}
\begin{proof}
Recall from \cref{eq:ethiopia-objective-f} that
\[
f(\bS) = \sum_{t=1}^h \sum_{c \in \textsc{covered}(\bS^{(1:t)})} w^{(t)}_c
\]
Since population counts are non-negative, we see that $f(\bA) \leq f(\bB)$ for any two selections $\bA \subseteq \bB$, so $f$ is non-decreasing.
Meanwhile, to see that $f$ is submodular, fix any $\bA \subseteq \bB \subseteq \bV$ and $e \in \bV \setminus \bB$.
Let $\tau \in [h]$ be the time step for which $e \in \bV^{(\tau)}$.
Then,
\begin{align*}
f(\bA \cup \{e\}) - f(\bA)
&= \sum_{\tau \leq t \leq h} \bigg( \sum_{c \in \textsc{covered}(\{e\} \setminus \bA^{(1:t)})} w^{(t)}_c \bigg) \tag{Definition of $f$}\\
&\geq \sum_{\tau \leq t \leq h} \bigg( \sum_{c \in \textsc{covered}(\{e\} \setminus \bB^{(1:t)})} w^{(t)}_c \bigg) \tag{Since $\bA \subseteq \bB$}\\
&\geq f(\bB \cup \{e\}) - f(\bB) \tag{Definition of $f$}
\end{align*}
Since $f(\bA \cup \{e\}) - f(\bA) \geq f(\bB \cup \{e\}) - f(\bB)$, we see that $f$ is submodular.
\end{proof}

\subsection{Noisy estimate of the true objective function}
\label{sec:appendix-proxy-function}

Suppose we do not have access to the true objective function $f$ in \OurProblem{} but are instead given a proxy function $\wt{f}$ such that
$
(1 - \eps) \cdot f(\bS) \leq \wt{f}(\bS) \leq (1 + \eps) \cdot f(\bS)
$
for all $\bS \subseteq \bV$, for some multiplicative error $\eps \in [0,1]$.

Given an algorithm $\ALG$ that has provable guarantees for \OurProblem{}, the next theorem shows that one can obtain provable approximation guarantees with respect to $\eps$ by running $\ALG$ in a blackbox manner while treating $\wt{f}$ as the true objective function.
In particular, the $\alpha$ in \cref{thm:proxy-function} is $\frac{1}{2} \cdot \frac{k}{k+1}$ if at least $k$ elements of each type is chosen per round when applying \OurAlg{}.

\begin{restatable}{theorem}{proxyfunction}
\label{thm:proxy-function}
Suppose we do not have access to $f$ for an $\OurProblem{}(f, h, r, \bV, \bb, \bp)$ instance.
Instead, we are given a proxy function $\wt{f}: 2^{\bV} \to \R$ with \emph{known} bounded multiplicative error $\eps \in [0,1]$ such that $(1 - \eps) \cdot f(\bS) \leq \wt{f}(\bS) \leq (1 + \eps) \cdot f(\bS)$ for all $\bS \subseteq \bV$.
If $\ALG(f)$ is an algorithm that achieves $\alpha$-approximation to $\OurProblem{}(f, h, r, \bV, \bb, \bp)$, for some $\alpha \in [0,1]$, then $\ALG(\frac{\wt{f}}{1 - \eps})$ achieves an approximation ratio of $\alpha \cdot \frac{1 - \eps}{1 + \eps}$ to $\OurProblem{}(f, h, r, \bV, \bb, \bp)$.
\end{restatable}
\begin{proof}
Besides the objective function $f$, let us fix all other parameters $h$, $q$, $\bV$, $\bb$, and $\bp$ arbitrarily.
So, for notational convenience, we simply write $\OurProblem{}(f)$ to mean $\OurProblem{}(f, h, q, \bV, \bb, \bp)$ in the rest of this proof.

For any function $f$, let $\bS_{\OPT}(f)$ denote the optimum type-feasible selection for $\OurProblem{}(f)$ and $\bS_{\ALG}(f)$ denote type-feasible selection produced by $\ALG(f)$.
For any type-feasible selection $\bS \subseteq \bV$, we write $\obj(\bS, \OurProblem{}(f))$ to denote the objective value attained by $\bS$ with respect to the objective function $f$.
Using this notation, we have
\begin{align*}
\frac{1}{\alpha} \cdot \obj(\bS_{\ALG}(f), \OurProblem{}(f))
&\geq \obj(\bS_{\OPT}(f), \OurProblem{}(f)) \tag{Since $\ALG$ is $\alpha$-approximate}\\
&= \max_{\substack{\bS \subseteq \bV\\ \text{$\bS$ is type-feasible}}} \obj(\bS, \OurProblem{}(f)) \tag{By optimality of $\bS_{\OPT}$}
\end{align*}
Therefore, using the proxy function $\wt{f}$, we see that
\begin{align*}
\alpha \cdot \obj(\bS_{\OPT}(f), \OurProblem{}(f))
&\leq \alpha \cdot \obj \left( \bS_{\OPT} (f), \OurProblem{} \left( \frac{\wt{f}}{1 - \eps} \right) \right) \tag{Since $(1 - \eps) \cdot f \leq \wt{f}$}\\
&\leq \alpha \cdot \obj \left( \bS_{\OPT} \left( \frac{\wt{f}}{1 - \eps} \right), \OurProblem{} \left( \frac{\wt{f}}{1 - \eps} \right) \right) \tag{Definition of optimum selection}\\
&\leq \obj \left( \bS_{\ALG} \left( \frac{\wt{f}}{1 - \eps} \right), \OurProblem{} \left( \frac{\wt{f}}{1 - \eps} \right) \right) \tag{$\ALG$ is $\alpha$-approximate}\\
&\leq \obj \left( \bS_{\ALG} \left( \frac{\wt{f}}{1 - \eps} \right), \OurProblem{} \left( \frac{1 + \eps}{1 - \eps} \cdot f \right) \right) \tag{Since $\wt{f} \leq (1 + \eps) \cdot f$}\\
&= \obj \left( \bS_{\ALG} (\wt{f}), \OurProblem{} \left( \frac{1 + \eps}{1 - \eps} \cdot f \right) \right) \tag{$\eps$ is an unknown \emph{constant}}\\
&= \frac{1 + \eps}{1 - \eps} \cdot \obj \left( \bS_{\ALG} (\wt{f}), \OurProblem{} \left( \frac{1 + \eps}{1 - \eps} \cdot f \right) \right) \tag{Since $\wt{f} \leq (1 + \eps) \cdot f$}
\end{align*}
That is, running $\ALG$ on the proxy function $\wt{f}$ yields an approximation ratio of $\alpha \cdot \frac{1 - \eps}{1 + \eps}$.
\end{proof}

\section{Deferred technical details for \texorpdfstring{\cref{sec:single-step}}{Section 4.1}}
\label{sec:appendix-single-step}

In this section, we first prove \cref{thm:LA-one-type-guarantee} where there is only $k = 1$ type, i.e., we are essentially solving size-constrained submodular maximization, and then show how to adapt to the setting with $k \geq 1$ types \cref{thm:LA-many-type-guarantee}.
For convenience, we reproduce \cref{alg:LA-one-type} below.

\begin{algorithm}[htb]
\caption{Learning-augmented algorithm for size-constrained non-decreasing submodular maximization}
\label{alg:LA-one-type-reproduced}
\begin{algorithmic}[1]
\Statex \textbf{Input}: Elements $\bV$, non-decreasing submodular set function $f: \bV \to \R$ to maximize for, budget $b \geq 1$, advice selection $\bA \subseteq \bV$ of size $|\bA| = b$
\Statex \textbf{Output}: A selection $\bU \subseteq \bV$
\State Define subsets $\bA_0, \bA_1, \ldots, \bA_b$ of $\bA$ where $|\bA_i| = i$ for sizes $i \in \{0, 1, \ldots, b\}$ \Comment{$\bA_0 = \emptyset$ and $\bA_b = \bA$}
\For{$i \in \{0, 1, \ldots, b\}$}
    \State Initialize $\bB_i = \emptyset$
    \While{$|\bB_i| < b - i$}
        \State Add $e^* = \argmax\limits_{e \in \bV \setminus (\bA_i \cup \bB_i)} g(\bA_i \cup \bB_i, \{e\})$ to $\bB_i$
    \EndWhile
    \State Define $\bU_i = \bA_i \uplus \bB_i$ \Comment{$|\bU_i| = |\bA_i| + |\bB_i| = b$}
\EndFor
\State \Return $\bU_{i^*}$, where $i^* = \argmax_{i \in \{0, 1, \ldots, b\}} f(\bU_i)$
\end{algorithmic}
\end{algorithm}

The proof of \cref{thm:LA-one-type-guarantee} follows from the following two lemmas: \cref{lem:LA-one-type-optimality} and \cref{lem:LA-one-type-guarantee-for-any-subset}.

\begin{lemma}
\label{lem:LA-one-type-optimality}
Fix any optimal selection $\OPT \subseteq \bV$ of size $|\OPT| = k$ and any arbitrary selection $\bS \subseteq \bV$ of size $|\bS| \leq k$.
Define $\bB$ as the best selection in $\OPT \setminus \bS$ such that
\[
f(\bB \cup \bS) = \max_{\substack{\bZ \subseteq \OPT \setminus \bS\\ |\bZ| = k - |\bS|}} f(\bZ \cup \bS)
\]
Then, for any non-empty subset $\bC \subseteq \OPT \setminus \bS$, we have
\[
f(\bC \cup \bS) - f(\bS) \leq \left\lceil \frac{|\bC|}{|\bB|} \right\rceil \cdot \left( f(\bB \cup \bS) - f(\bS) \right)
\]
\end{lemma}
\begin{proof}
Break up $\bC$ into $\ell = \left\lceil \frac{|\bC|}{|\bB|} \right\rceil$ partitions $\bC = \bC_1 \uplus \ldots \uplus \bC_{\ell}$, each of size at most $|\bB|$.
Then,
\begin{align*}
f(\bC \cup \bS) - f(\bS)
&\leq \sum_{i=1}^{\ell} \bigg( f(\bC_i \cup \bS) - f(\bS) \bigg) \tag{By submodularity of $f$}\\
&\leq \sum_{i=1}^{\ell} \bigg( f(\bB \cup \bS) - f(\bS) \bigg)\tag{By maximality of $\bB$}\\
&= \left\lceil \frac{|\bC|}{|\bB|} \right\rceil \cdot \bigg( f(\bB \cup \bS) - f(\bS) \bigg) \tag{Since $\ell = \left\lceil \frac{|\bC|}{|\bB|} \right\rceil$}
\end{align*}
The claim follows by rearranging the terms.
\end{proof}

\begin{lemma}
\label{lem:LA-one-type-guarantee-for-any-subset}
Fix any optimal selection $\OPT \subseteq \bV$ and advice selection $\bA \subseteq \bV$ of size $|\OPT| = |\bA| = k$.
Fix an arbitrary subset $\bA' \subseteq \bA$ of $\bA$.
Let $\bB' \subseteq \bV \setminus \bA'$ be the output produced from greedy selection and $\OPT' \subseteq \bV \setminus \bA'$ be the optimum selection with respect to the set function $f(\bS \cup \bA') - f(\bA')$ for any subset $\bS \subseteq \bV \setminus \bA'$, where $|\bB'| = |\OPT'| = k - |\bA'|$.
Then,
\[
f(\bB' \cup \bA')
\geq f(\bA')
+ \frac{1 - \frac{1}{e}}{\left\lceil \frac{|\OPT \setminus \bA'|}{k - |\bA'|} \right\rceil} \cdot \bigg( f(\OPT \cup \bA') - f(\bA') \bigg)
\]
\end{lemma}
\begin{proof}
Define $\bB \subseteq \OPT \setminus \bA'$ such that
\[
f(\bB \cup \bA') = \max_{\substack{\bZ \subseteq \OPT \setminus \bA'\\ |\bZ| = k - |\bA'|}} f(\bZ \cup \bA')
\]
Then, we have
\begin{align*}
f(\OPT \cup \bA') - f(\bA')
&\leq \left\lceil \frac{|\OPT \setminus \bA'|}{k - |\bA'|} \right\rceil \cdot \bigg( f(\bB \cup \bA') - f(\bA') \bigg) \tag{$\ast$}\\
&\leq \left\lceil \frac{|\OPT \setminus \bA'|}{k - |\bA'|} \right\rceil \cdot \bigg( f(\OPT' \cup \bA') - f(\bA') \bigg) \tag{Since $\OPT \setminus \bA' \subseteq \bV \setminus \bA'$ and by optimality of $\OPT'$}\\
&\leq \frac{1}{1 - \frac{1}{e}} \cdot \left\lceil \frac{|\OPT \setminus \bA'|}{k - |\bA'|} \right\rceil \cdot \bigg( f(\bB' \cup \bA') - f(\bA') \bigg) \tag{$\ddag$}
\end{align*}
where $(\ast)$ is via \cref{lem:LA-one-type-optimality} while $(\ddag)$ is due to the greedy construction of $\bB'$ while satisfying the size constraint $k$; see \cite{nemhauser1978analysis}.
The claim follows by rearranging the terms.
\end{proof}

\LAonetypeguarantee*
\begin{proof}
By \cref{lem:LA-one-type-guarantee-for-any-subset}, we have
\[
f(\bU_i)
\geq f(\bA_i)
+ \frac{1 - \frac{1}{e}}{\left\lceil \frac{|\OPT \setminus \bA_i|}{k-i} \right\rceil} \cdot \bigg( f(\OPT \cup \bA_i) - f(\bA_i) \bigg)
\]
for all $i \in \{0, 1, \ldots, k\}$.
The claim follows since $\bU$ is the best selection out of $\bU_0, \bU_1, \ldots, \bU_k$.
\end{proof}

To extend \cref{thm:LA-one-type-guarantee} to $r \geq 1$ types, consider an alternative objective set function $f'$ which computes $f$ additively across each type: $f'(\bS) = \sum_{q=1}^r f(\bS \cap \bT_q)$ for any $\bS \subseteq \bV$.
By construction, $f'$ never underestimates $f$.
That is, $f'(\bS) \geq f(\bS) \geq \beta_{f, \bS} \cdot f'(\bS)$ for some $\beta_{f, \bS} \in [0,1]$ that depends on $f$ and $\bS$.
Operating under the assumption that $f$ decomposes with respect to the types simplifies analysis while worsening the approximation ratio by a multiplicative factor proportional to $\max_{\bS \subseteq \bV} \beta_{f, \bS}$; we can invoke the result of \cref{sec:appendix-proxy-function} to $f'$.
Using $f'$ as a proxy of $f$ is justified because each facility primarily serves the population of that same district.
For instance, in the Sidama region of Ethiopia, an average of 81\% of the population within a two-hour walking distance from a cell belongs to the same district.

\begin{restatable}{theorem}{LAmanytypeguarantee}
\label{thm:LA-many-type-guarantee}
Suppose the objective $f$ is type-decomposable and we need to produce a selection subject to a partition matroid over the types with $\cI = \{ \bS \subseteq \bV: | \bS \cap \bT_q | \leq x_q, \forall q \in [r]\}$, for some $x_1, \ldots, x_q \in \N$.
Let $\OPT \subseteq \bV$ and $\bA \subseteq \bV$ be any optimal and advice selection such that $|\OPT \cap \bT_q| = |\bA \cap \bT_q| = x_q$ for all types $q \in [r]$.
Then, for any subset $\bA' \subseteq \bA$, there is a polynomial time algorithm that producing $\bB' \subseteq \bV \setminus \bA'$ such that $\bB' \cup \bA' \in \cI$ and
\[
f(\bB' \cup \bA')
\geq f(\bA')
+ \frac{1 - \frac{1}{e}}{\max\limits_{q \in [r]} \left\lceil \frac{|(\OPT \setminus \bA') \cap \bT_q|}{x_q - |\bA' \cap \bT_q|} \right\rceil} \cdot \bigg( f(\OPT \cup \bA') - f(\bA') \bigg)
\]
\end{restatable}
\begin{proof}
Define $\bB \subseteq \OPT \setminus \bA$ such that $\bB \cup \bA \in \cI$ and
\[
f(\bB \cup \bA') = \max_{\substack{\bZ \subseteq \OPT \setminus \bA'\\ \bZ \cup \bA' \in \cI}} f(\bZ \cup \bA')
\]
Since $f$ decomposes additively across types, we get
\[
f \big( (\bB \cup \bA') \cap \bT_q \big) = \max_{\substack{\bZ \subseteq \OPT \setminus \bA'\\ |\bZ \cap \bT_q| \leq x_q - |\bA' \cap \bT_q|}} f \big( (\bZ \cup \bA') \cap \bT_q \big)
\]
for each type $q \in [r]$.
We define $\bB'$ to be the greedy selection produced by \textsc{LocalGreedy} (\cref{alg:localgreedy}) over the matroid $\cI$ having already selected $\bA'$.
For readability, let us define
\[
\Box_q = \left\lceil \frac{|(\OPT \setminus \bA') \cap \bT_q|}{x_q - |\bA' \cap \bT_q|} \right\rceil
\]
for each type $q \in [r]$.
Then, we have
\begin{align*}
f(\OPT \cup \bA') - f(\bA')
&= \sum_{q=1}^r f \big( (\OPT \cup \bA') \cap \bT_q \big) - f \big(\bA' \cap \bT_q \big) \tag{Since $f$ decomposes additively across types}\\
&\leq \sum_{q=1}^r \Box_q \cdot \bigg( f \big( (\bB \cup \bA') \cap \bT_q \big) - f \big(\bA' \cap \bT_q \big) \bigg) \tag{$\ast$}\\
&\leq \left( \max_{q \in [r]} \Box_q \right) \cdot \bigg( f(\bB \cup \bA') - f(\bA') \bigg) \tag{Since $f$ decomposes additively across types}\\
&\leq \left( \max_{q \in [r]} \Box_q \right) \cdot \bigg( f(\OPT' \cup \bA') - f(\bA') \bigg) \tag{Since $\OPT \setminus \bA' \subseteq \bV \setminus \bA'$ and by optimality of $\OPT'$}\\
&\leq \frac{1}{1 - \frac{1}{e}} \cdot \left( \max_{q \in [r]} \Box_q \right) \cdot \bigg( f(\bB' \cup \bA') - f(\bA') \bigg) \tag{$\ddag$}
\end{align*}
where $(\ast)$ is via \cref{lem:LA-one-type-optimality} while $(\ddag)$ is due to the greedy construction of $\bB'$ while satisfying the matroid constraint $\cI$; see \cite{fisher1978analysis}.
The claim follows by rearranging the terms.
\end{proof}

Motivated by our analyses, one can define the following learning-augmented variant of \OurAlg{} in \cref{alg:ouralg-with-advice}.
While one can use similar approaches to obtain theoretical guarantees for \cref{alg:ouralg-with-advice}, the resulting guarantees are unwieldly and do not provide any intuition.
As such, we do not make any theoretical claims and simply present \cref{alg:ouralg-with-advice} as a heuristic algorithm.

\begin{algorithm}[htb]
\caption{\OurAlgWithAdvice{}}
\label{alg:ouralg-with-advice}
\begin{algorithmic}[1]
\Statex \textbf{Input}: Objective function $f$, time horizon $h$, elements $\bV$ over $r$ types, budgets $b^{(1)}, \ldots, b^{(h)} \in \N$ that arrive over time, advice selections $\bA^{(1)}, \ldots, \bA^{(h)}$ that arrive over time, proportions $p_1, \ldots, p_r \in [0,1]$ for each type, tie-breaking ordering $\sigma: [r] \to [r]$
\Statex \textbf{Output}: A selection $\bS = \bS^{(1)} \uplus \ldots \uplus \bS^{(h)}$
\State Define $\bS^{(1)} = \ldots = \bS^{(h)} = \emptyset$ \Comment{Initialize selections}
\For{$t = 1, \ldots, h$}
    \State For $q \in [r]$, define $\#(q, \bs(b^{(1:0)}, b^{(1:0)})) = 0$ and
    \[
    x^{(t)}_q = \#(q, \bs(b^{(1:t)}, b^{(1:t)})) - \#(q, \bs(b^{(1:t-1)}, b^{(1:t-1)}))
    \]
    \State Let $\cM^{(t)} = (\bV^{(t)}, \cI^{(t)})$ be the partition matroid at
    \Statex\hspace{\algorithmicindent}time $t$, where $\cI^{(t)} = \{ \bS \subseteq \bV^{(t)}: | \bS \cap \bT_q | \leq x^{(t)}_q \}$.
    \State Define subsets $\bA^{(t)}_0, \bA^{(t)}_1, \ldots, \bA^{(t)}_{b^{(t)}}$ of $\bA^{(t)}$ where
    \Statex\hspace{\algorithmicindent}$|\bA^{(t)}_i| = i$ for sizes $i \in \{0, 1, \ldots, b^{(t)}\}$
    \For{$i \in \{0, 1, \ldots, b^{(t)}\}$}
        \State Initialize $\bB^{(t)}_i = \emptyset$
        \While{$|\bB^{(t)}_i| < b^{(t)} - i$}
            \State Add $e^*$ to $\bB^{(t)}_i$, where
            \[
            e^* =
            \argmax\limits_{\substack{e \in \bV^{(t)} \setminus (\bA^{(t)}_i \cup \bB^{(t)}_i)\\ \bS^{(1:t-1)} \cup \bA^{(t)}_i \cup \bB^{(t)}_i \cup \{e\} \in \cI^{(t)}}} g(\bS^{(1:t-1)} \cup \bA^{(t)}_i \cup \bB^{(t)}_i, \{e\})
            \]
        \EndWhile
        \State Define $\bU^{(t)}_i = \bA^{(t)}_i \uplus \bB^{(t)}_i$ \Comment{$|\bU^{(t)}_i| = b^{(t)}$}
    \EndFor
    \State Define $\bS^{(t)} = \bU^{(t)}_{i^*}$, where
    \[
    i^* = \argmax\limits_{i \in \{0, 1, \ldots, b^{(t)}\}} f(\bU^{(t)}_i)
    \]
\EndFor
\State Output $\bS^{(1)} \uplus \ldots \uplus \bS^{(h)}$
\end{algorithmic}
\end{algorithm}

\section{Deferred technical details for \texorpdfstring{\cref{sec:multi-step}}{Section 4.2}}
\label{sec:appendix-multi-step}

In this section, we work towards proving \cref{thm:apx-ratio-of-our-algo}.
We begin by proving \cref{thm:no-sigma}, along with some additional results about $\sigma$-type feasible selections (\cref{def:sigma-type-feasible-selection}) that will be useful for proving algorithmic guarantees subsequently.

Note that any $\sigma$-type feasible selection (\cref{def:sigma-type-feasible-selection}) is also type feasible (\cref{def:type-feasible-selection}), though the latter can tie-break arbitrarily when selecting items from the set of types with minimal satisfaction ratio.
We now show that a $\sigma$-type feasible optimum selection $\OPT_\sigma$ is competitive against a type feasible optimum selection $\OPT$ if at least $k$ elements of each type is chosen per round when applying \OurAlg{}.

The following lemma about $\sigma$-type feasible subsets of a type feasible selections will be useful for proving our main result for this section (\cref{thm:no-sigma}).

\begin{restatable}{lemma}{tiebrokentypefeasiblesubsetoftypefeasibleselection}
\label{lem:sigma-type-feasible-subset-of-type-feasible-selection}
Fix an \OurProblem{} problem instance $\cP$.
Let $\OPT_\sigma$ be a $\sigma$-type feasible optimal selection to $\cP$ and $\OPT$ be a type feasible optimal selection to $\cP$.
Let $\cI \subseteq 2^{\bV}$ be the independence set corresponding to the tie-broken type feasibility of $\OPT_\sigma$ and define $\bW \subseteq \OPT$ to minimize the marginal gain:
\[
g(\OPT \setminus \bW, \bW)
= \min_{\substack{\bA \subseteq \OPT \subseteq \bV\\ (\OPT \setminus \bA) \in \cI}} g(\OPT \setminus \bA, \bA)
\]
Then, $|\bW \cap \bV^{(t)} \cap \bT_q| \leq 1$ for any $t \in [h]$ and $q \in [r]$.
\end{restatable}
\begin{proof}
Note that the minimality of $\bW$ would try to make $\OPT \setminus \bW$ overlap with $\OPT_\sigma$, in terms of element types at each time step, as much as possible.
More formally, for any time step $t \in [h]$ and type $q \in [r]$, we have
\begin{equation}
\label{eq:minimality-of-W}
|(\OPT \setminus \bW) \cap \bV^{(t)} \cap \bT_q|\\
= \min \bigg\{|\OPT_\sigma \cap \bV^{(t)} \cap \bT_q|, |\OPT \cap \bV^{(t)} \cap \bT_q| \bigg\}
\end{equation}

\medskip
\noindent
\textbf{Case 1:} $|\OPT_\sigma \cap \bV^{(t)} \cap \bT_q| > |\OPT \cap \bV^{(t)} \cap \bT_q|$.\\
By \cref{eq:minimality-of-W}, we have $|(\OPT \setminus \bW) \cap \bV^{(t)} \cap \bT_q| = |\OPT \cap \bV^{(t)} \cap \bT_q|$.
So,
\begin{align*}
|\bW \cap \bV^{(t)} \cap \bT_q|
&= |\OPT \cap \bV^{(t)} \cap \bT_q| - |(\OPT \setminus \bW) \cap \bV^{(t)} \cap \bT_q|\\
&= |\OPT \cap \bV^{(t)} \cap \bT_q| - |\OPT \cap \bV^{(t)} \cap \bT_q| \tag{By \cref{eq:minimality-of-W}}\\
&= 0 \leq 1
\end{align*}

\medskip
\noindent
\textbf{Case 2:} $|\OPT_\sigma \cap \bV^{(t)} \cap \bT_q| \leq |\OPT \cap \bV^{(t)} \cap \bT_q|$.\\
Suppose, for a contradiction, that $|\bW \cap \bV^{(t)} \cap \bT_q| \geq 2$.
By \cref{eq:minimality-of-W}, we have $|(\OPT \setminus \bW) \cap \bV^{(t)} \cap \bT_q| = |\OPT_\sigma \cap \bV^{(t)} \cap \bT_q|$.
So,
\begin{align*}
|\OPT \cap \bV^{(t)} \cap \bT_q|
&= |\bW \cap \bV^{(t)} \cap \bT_q| + |(\OPT \setminus \bW) \cap \bV^{(t)} \cap \bT_q|\\
&\geq 2 + |(\OPT \setminus \bW) \cap \bV^{(t)} \cap \bT_q| \tag{Since $|\bW \cap \bV^{(t)} \cap \bT_q| \geq 2$}\\
&= 2 + |\OPT_\sigma \cap \bV^{(t)} \cap \bT_q| \tag{By \cref{eq:minimality-of-W}}
\end{align*}

Let $e$ be an arbitrary element in $\bW \cap \bV^{(t)} \cap \bT_q$.
Note that $\bW \subseteq \OPT$, $\alpha_{\min}((\OPT \cap \bV^{(1:t)}) \setminus \{e\}) \leq \alpha_{\min}(\OPT \cap \bV^{(1:t)})$, and the inequality is strict only when removing $e$ causes the min-ratio to worsen.
As $\OPT_\sigma$ and $\OPT$ are both type feasible, we have $\alpha_{\min}(\OPT_\sigma \cap \bV^{(1:t)}) = \alpha_{\min}(\OPT \cap \bV^{(1:t)})$ for any time step $t \in [h]$, and so
\[
\bigg| |\OPT_\sigma \cap \bV^{(1:t)} \cap \bT_q| - |\OPT \cap \bV^{(1:t)} \cap \bT_q| \bigg| \leq 1
\]
for all $t \in [h]$ and $q \in [r]$.
In particular, we have
\[
\bigg| |\OPT_\sigma \cap \bV^{(1:t-1)} \cap \bT_q| - |\OPT \cap \bV^{(1:t-1)} \cap \bT_q| \bigg| \leq 1
\]
at time step $t-1$ and any $q \in [r]$.
Now, since we have $|\OPT \cap \bV^{(t)} \cap \bT_q| \geq 2 + |\OPT_\sigma \cap \bV^{(t)} \cap \bT_q|$ from above, so it must be the case that removing $e$ of $\type(e) = q$ does \emph{not} worsen the min-ratio.
So,
\begin{equation}
\label{eq:alpha-min-unchanged-after-removal-of-one-element}
\alpha_{\min}((\OPT \cap \bV^{(1:t)}) \setminus \{e\})
= \alpha_{\min}(\OPT \cap \bV^{(1:t)})
\end{equation}

Let $e_{\min} \in \bV^{(1:t)} \setminus \OPT$ be an element with $\type(e_{\min}) \in \bQ(\OPT \cap \bV^{(1:t)})$.
Then, the set $\bS = (\OPT \cap \bV^{(1:t)}) \setminus \{e\} \cup \{e_{\min}\}$ satisfies $|\bS| = |\OPT \cap \bV^{(1:t)}|$ and
\begin{align*}
\alpha_{\min}(\bS)
&\; \geq \alpha_{\min}((\OPT \cap \bV^{(1:t)}) \setminus \{e\}) \tag{Since $(\OPT \cap \bV^{(1:t)}) \setminus \{e\} \subset \bS$}\\
&\; = \alpha_{\min}(\OPT \cap \bV^{(1:t)}) \tag{By \cref{eq:alpha-min-unchanged-after-removal-of-one-element}}
\end{align*}
We should not have $\alpha_{\min}(\bS) > \alpha_{\min}(\OPT \cap \bV^{(1:t)})$ since $\OPT$ is supposed to be type feasible.
However, if $\alpha_{\min}(\bS) = \alpha_{\min}(\OPT \cap \bV^{(1:t)})$, then we have $|\bQ(\bS)| < |\bQ(\OPT \cap \bV^{(1:t)})|$ because $\bS$ has the additional element $e_{\min}$ with $\type(e_{\min}) \in \bQ(\OPT \cap \bV^{(1:t)})$.
This contradicts the lemma assumption that $\OPT$ is type feasible.
So, we \emph{cannot} have $|\bW \cap \bV^{(t)} \cap \bT_q| \geq 2$, i.e., $|\bW \cap \bV^{(t)} \cap \bT_q| \leq 1$.
\end{proof}

\nosigma*
\begin{proof}
Let $\cI \subseteq 2^{\bV}$ be the independence set corresponding to the type feasibility of $\OPT_\sigma$ and define $\bW \subseteq \OPT$ to minimize the marginal gain:
\[
g(\OPT \setminus \bW, \bW)
= \min_{\substack{\bA \subseteq \OPT \subseteq \bV\\ (\OPT \setminus \bA) \in \cI}} g(\OPT \setminus \bA, \bA)
\]
Since $f$ is non-decreasing, the optimization will try to pick a $\bW$ that is as small as possible, subject to $\OPT \setminus \bW$ satisfying the independence set constraints of $\cI$.

Fix a time step $t \in [h]$ and type $q \in [r]$.
By minimality of $\bW$, we know that $|\bW \cap \bV^{(t)} \cap \bT_q| \leq |\OPT \cap \bV^{(t)} \cap \bT_q| - |\OPT_\sigma \cap \bV^{(t)} \cap \bT_q|$.
Putting together, we get
\begin{align*}
|(\OPT \setminus \bW) \cap \bV^{(t)} \cap \bT_q|
&= |\OPT \cap \bV^{(t)} \cap \bT_q| - |\bW \cap \bV^{(t)} \cap \bT_q|\\
&\geq |\OPT_\sigma \cap \bV^{(t)} \cap \bT_q| \tag{Minimality of $\bW$}\\
&\geq k \tag{By assumption}
\end{align*}
By \cref{lem:sigma-type-feasible-subset-of-type-feasible-selection}, we know that $|\bW \cap \bV^{(t)} \cap \bT_q| \leq 1$, and therefore there exist $k$ disjoint subsets $\bY_1, \ldots, \bY_k \subseteq (\OPT \setminus \bW) \cap \bV^{(t)} \cap \bT_q$ such that $|\bY_i \cap \bV^{(t)} \cap \bT_q| = |\bW \cap \bV^{(t)} \cap \bT_q|$ and $\OPT \setminus \bY_i \in \cI$ for each index $i \in [k]$.
Writing $\bZ = \OPT \setminus \bW$, observe that
\begin{align*}
f(\bZ)
&\; \geq f(\uplus_{i=1}^k \bY_i) \tag{$\ast$}\\
&\; = \sum_{i=1}^k g(\uplus_{j=1}^{i-1} \bY_j, \bY_i) \tag{By definition of $g$}\\
&\; \geq \sum_{i=1}^k g(\OPT \setminus \bY_i, \bY_i) \tag{$\ddag$}\\
&\; \geq \sum_{i=1}^k g(\bZ, \bW) \tag{Minimality of $\bW$}\\
&\; = k \cdot g(\bZ, \bW)
\end{align*}
where $(\ast)$ is because $\uplus_{i=1}^k \bY_i \subseteq \bZ$ and $f$ is non-decreasing, and $(\ddag)$ is because $\uplus_{j=1}^{i-1} \bY_j \subseteq \OPT \setminus \bY_i$ and $f$ is submodular.
Therefore,
\begin{align*}
f(\OPT)
&\; = f(\bZ) + g(\bZ, \bW) \tag{By definition of $g$ and $\bZ$}\\
&\; \leq f(\bZ) + \frac{1}{k} \cdot f(\bZ) \tag{From above}\\
&\; \leq \frac{k+1}{k} \cdot f(\OPT_\sigma)
\end{align*}
where the last inequality is due to optimality of $\OPT_\sigma$ and since $\bZ = \OPT \setminus \bW \in \cI$ by choice of $\bW$.
That is, $\bZ$ satisfies the type feasible constraints of $\OPT_\sigma$.
\end{proof}

We now prove some additional properties about $\sigma$-type feasible selections (\cref{def:sigma-type-feasible-selection}) that will be useful in our proofs in later sections.

\begin{definition}[$b$-minimal $\sigma$-type feasible set]
For any $b \in \N$, we say that a subset $\bS \subseteq \bV$ is a $b$-minimal $\sigma$-type feasible set if $1 \leq |\bS| \leq b$ and
\[
\min_{q \in [r]} \frac{|\bS \cap \bT_q|}{p_q \cdot b}
> \max_{e \in \bS} \min_{q \in [r]} \frac{|(\bS \setminus \{e\} \cap \bT_q|}{p_q \cdot b}
\geq 0
\]
\end{definition}

\begin{lemma}
\label{lem:minimal-containment}
Let $\bS$ be a $\sigma$-type feasible set of size $|\bS| = b$.
Then, we have $\bS' \subseteq \bS$ for any $b$-minimal $\sigma$-type feasible set $\bS'$ such that
\begin{equation}
\label{eq:choice-of-s-prime}
\min_{q \in [r]} \frac{|\bS' \cap \bT_q|}{p_q \cdot b}
\leq \min_{q \in [r]} \frac{|\bS \cap \bT_q|}{p_q \cdot b}
\end{equation}
\end{lemma}
\begin{proof}
Suppose, for a contradiction, that there is some $(a', b)$-minimal $\sigma$-type feasible set $\bS'$ such that $\bS' \not\subseteq \bS$.
Then, there exists some element $e \in \bS' \setminus \bS$ such that $q^* = \type(e)$ and $|\bS' \cap \bT_{q^*}| > |\bS \cap \bT_{q^*}|$.
So, $|(\bS' \setminus \{e\}) \cap \bT_{q^*}| \geq |\bS \cap \bT_{q^*}|$.
Then, we see that
\begin{align*}
\min_{q \in [r]} \frac{|\bS \cap \bT_q|}{p_q \cdot b}
&\leq \frac{|\bS \cap \bT_{q^*}|}{p_{q^*} \cdot b} \tag{By definition of minimum}\\
&\leq \frac{|(\bS' \setminus \{e\}) \cap \bT_{q^*}|}{p_{q^*} \cdot b} \tag{From above}\\
&< \min_{q \in [r]} \frac{|\bS' \cap \bT_q|}{p_q \cdot b} \tag{Since $\bS'$ is $b$-minimal}
\end{align*}
However, this is a contradiction to \cref{eq:choice-of-s-prime}.
\end{proof}

\cref{lem:sigma-type-feasible-is-extendable} implies that any $\sigma$-type feasible solution of length $b$ is an extension of another $\sigma$-type feasible solution of length $b-1$.

\begin{lemma}
\label{lem:sigma-type-feasible-is-extendable}
For any length $b \in \N_{> 0}$, let $\bS_b$ be a $\sigma$-type feasible selection of length $b$ and $\bS_{b-1}$ be a $\sigma$-type feasible selection of length $b-1$.
Then, $|\bS_b \setminus \bS_{b-1}| = 1$.
\end{lemma}
\begin{proof}
We first show that
\begin{equation}
\label{eq:extendable-ratio-inequality}
\min_{q \in [r]} \frac{|\bS_{b-1} \cap \bT_q|}{p_q \cdot b} \leq \min_{q \in [r]} \frac{|\bS_b \cap \bT_q|}{p_q \cdot b}
\end{equation}
Suppose, for a contradiction, that \cref{eq:extendable-ratio-inequality} is false.
Then, any strict superset of $\bA \supset \bS_{b-1}$ of size $b$ will have
\[
\min_{q \in [r]} \frac{|\bA \cap \bT_q|}{p_q \cdot b} 
\geq \min_{q \in [r]} \frac{|\bS_{b-1} \cap \bT_q|}{p_q \cdot b} 
> \min_{q \in [r]} \frac{|\bS_b \cap \bT_q|}{p_q \cdot b}
\]
since adding elements can only increase the numerator.
This contradicts the assumption that $\bS_b$ is a $\sigma$-type feasible selection of size $b$, so \cref{eq:extendable-ratio-inequality} must be true.

\bigskip
Now, suppose that $\min_{q \in [r]} \frac{|\bS_{b-1} \cap \bT_q|}{p_q \cdot b} = 0$.
Then,
\[
\alpha_{\min}(\bS_{b-1}) = \frac{b}{b-1} \cdot \min_{q \in [r]} \frac{|\bS_{b-1} \cap \bT_q|}{p_q \cdot b} = 0
\]
Since $\bS_{b-1}$ is $\sigma$-type feasible, it must be the case that $\bS_{b-1} \subseteq \bQ(\emptyset)$ and $\bS_{b-1}$ comprises of elements following the $\sigma$ ordering.
If $\alpha_{\min}(\bS_b) = 0$ as well, then so does $\bS_b$.
Otherwise, if $\alpha_{\min}(\bS_b) > 0$, then we must have $\bQ(\emptyset) \subseteq \bS_b$.
Both cases implies that $\bS_{b-1} \subset \bS_b$ and so $|\bS_b \setminus \bS_{b-1}| = 1$.
Note that we cannot have $\alpha_{\min}(\bS_{b-1}) > \alpha_{\min}(\bS_b)$ due to \cref{eq:extendable-ratio-inequality}.

Given the argument above, we may assume that $\min_{q \in [r]} \frac{|\bS_{b-1} \cap \bT_q|}{p_q \cdot b} > 0$ in the remaining of the proof.
This implies that there will exist non-empty $b$-minimal $\sigma$-type feasible sets which is a necessary condition for us to invoke \cref{eq:extendable-ratio-inequality}.

\bigskip
Suppose $\bS_{b-1}$ is a $b$-minimal $\sigma$-type feasible set.
Then, under \cref{eq:extendable-ratio-inequality}, \cref{lem:minimal-containment} yields $\bS_{b-1} \subseteq \bS$, so $|\bS_b \setminus \bS_{b-1}| = 1$.

\bigskip
Now, suppose that $\bS_{b-1}$ is \emph{not} a $b$-minimal $\sigma$-type feasible set.

\noindent
\textbf{Case 1}:
Suppose $\min_{q \in [r]} \frac{|\bS_{b-1} \cap \bT_q|}{p_q \cdot b} = \min_{q \in [r]} \frac{|\bS_b \cap \bT_q|}{p_q \cdot b}$.

Let $\bA$ be a $b$-minimal $\sigma$-type feasible subset such that
\[
\min_{q \in [r]} \frac{|\bA \cap \bT_q|}{p_q \cdot b} = \min_{q \in [r]} \frac{|\bS \cap \bT_q|}{p_q \cdot b}
\stackrel{\text{Case 1}}{=}
\min_{q \in [r]} \frac{|\bS' \cap \bT_q|}{p_q \cdot b}
\]
By \cref{lem:minimal-containment}, we have both $\bA \subseteq \bS_{b-1}$ and $\bA \subseteq \bS_b$.
Since $\bS_{b-1}$ is $\sigma$-type feasible, it maximizes $\beta(b-1)$ where
\[
\beta(b-1)
= \min_{q \in [r]} \frac{|\bS_{b-1} \cap \bT_q|}{p_q \cdot (b-1)}
= \frac{b}{b-1} \cdot \min_{q \in [r]} \frac{|\bA \cap \bT_q|}{p_q \cdot b}
\]
Since $\frac{b}{b-1}$ is a constant independent of $q \in [r]$, we see that $\bA$ also maximizes $\beta(b-1)$.
Since a $\sigma$-type feasible selection must minimize $|\bQ(\bS)|$ and $\bA$ is a minimal subset, we see that $\bS' \setminus \bA \subset \bQ(\bA)$.
Crucially, note that $\bQ(\bA) \setminus \bS' \neq \emptyset$, otherwise $\min_{q \in [r]} \frac{|\bA \cap \bT_q|}{p_q \cdot b} < \min_{q \in [r]} \frac{|\bS' \cap \bT_q|}{p_q \cdot b}$.
Following a similar argument, one can show that $\bS_b \setminus \bA \subset \bQ(\bA)$.
Now, since both $\bS_{b-1}$ and $\bS_b$ are $\sigma$-type feasible, they must select from $\bQ(\bA)$ in the ordering of $\sigma$.
In other words, we have $\bS_{b-1} \setminus \bA \subseteq \bS_b \setminus \bA$, so $\bS_{b-1} \subseteq \bS_b$, and thus $|\bS_b \setminus \bS_{b-1}| = 1$.
\\

\noindent
\textbf{Case 2}:
Suppose $\min_{q \in [r]} \frac{|\bS_{b-1} \cap \bT_q|}{p_q \cdot b} < \min_{q \in [r]} \frac{|\bS_b \cap \bT_q|}{p_q \cdot b}$.

Let $\bB$ be a $b$-minimal $\sigma$-type feasible subset such that
\[
\min_{q \in [r]} \frac{|\bB \cap \bT_q|}{p_q \cdot b} = \min_{q \in [r]} \frac{|\bS' \cap \bT_q|}{p_q \cdot b}
\stackrel{\text{Case 2}}{<}
\min_{q \in [r]} \frac{|\bS \cap \bT_q|}{p_q \cdot b}
\]
By \cref{lem:minimal-containment}, we have both $\bB \subseteq \bS_{b-1}$ and $\bB \subseteq \bS_b$.
Following a similar argument as Case 1, we see that $\bS_{b-1} \setminus \bB \subset \bQ(\bB)$.
However, since $\min_{q \in [r]} \frac{|\bB \cap \bT_q|}{p_q \cdot b} < \min_{q \in [r]} \frac{|\bS_b \cap \bT_q|}{p_q \cdot b}$, we must have $\bQ(\bB) \subseteq \bS$.
In other words, we get $\bS_{b-1} \subseteq \bS_b$ and thus $|\bS_b \setminus \bS_{b-1}| = 1$.
\end{proof}

To operationalize the abstract requirements of $\sigma$-type feasible selections (\cref{def:sigma-type-feasible-selection}), let us define the notion of min-ratio type sequences (\cref{def:min-ratio-type-sequence}).

\begin{definition}[Min-ratio type sequence]
\label{def:min-ratio-type-sequence}
Let $\bp = (p_1, \ldots, p_r) \in \R_{\geq 0}^r$ be a vector of lower distributional bounds, $\sigma: [r] \to [r]$ be a total order over types, $b$ be a size and let $d \in \N$ be a fixed denominator. 
The \emph{min-ratio type sequence} $\bs(\bp, \sigma,b, d) \in [r]^\N$ is defined by starting with the empty sequence and repeating the following steps $b$ times: append to the end of $\bs(\bp, \sigma,b, d)$ the type $q \in [r]$ that minimizes $\frac{\#(q, \bs(\bp, \sigma,b, d))}{p_q \cdot d}$, breaking ties according to the preference ordering $\sigma$.
Since $p$ and $\sigma$ are fixed inputs to any problem instance, we simply write $\bs(b,d)$ to mean $\bs(\bp, \sigma, b, d)$.
\end{definition}

\cref{lem:min-ratio-type-sequence-is-independent-of-denominator} shows that \cref{def:min-ratio-type-sequence} is independent of the denominator, which will be useful in our analysis later.

\begin{restatable}{lemma}{minratiotypesequenceisindependentofdenominator}
\label{lem:min-ratio-type-sequence-is-independent-of-denominator}
For fixed $\bp \in \R_{\geq 0}^r$, $\sigma: [r] \to [r]$, and $b \in \N$, we have $\bs(b, d_1) = \bs(b, d_2)$ for any $d_1, d_2 \in \N_{> 0}$.
\end{restatable}
\begin{proof}
We will prove this by induction over $b \in \N$.

\noindent
\textbf{Base case} ($b = 0$):
When $b = 0$, we trivially have both $\bs(0, d_1)$ and $\bs(0, d_2)$ being the empty sequence.

\noindent
\textbf{Inductive case} ($b > 0$):
By induction hypothesis, we know that $\bs(b-1, d_1) = \bs(b-1, d_2)$, so $\#(q, \bs(b-1, d_1)) = \#(q, \bs(b-1, d_2))$ for any $q \in [r]$.
Since $\frac{\#(q, \bs(b-1, d_1))}{p_q \cdot d_1} = \frac{d_2}{d_1} \cdot \frac{\#(q, \bs(b-1, d_2))}{p_q \cdot d_2}$ for any $q \in [r]$ and $\frac{d_2}{d_1}$ is a constant independent of $q$, we have $\argmin_{q \in [r]} \frac{\#(q, \bs(b-1, d_1))}{p_q \cdot d_1} = \argmin_{q \in [r]} \frac{\#(q, \bs(b-1, d_2))}{p_q \cdot d_2}$, and thus both sequences will append the same $b$-th type, resulting in $\bs(b, d_1) = \bs(b, d_2)$.
\end{proof}

Observe that \cref{def:sigma-type-feasible-selection} and \cref{def:min-ratio-type-sequence} are both defined independently of the budget constraints across time.
The former is an abstract axiomatic definition of the desired properties motivated from our Ethiopian application while the latter is a specific concrete way of choosing a sequence of types.
Interestingly, \cref{lem:type-feasible-equal-sequence} tells us that multisets of types produced by the two definitions are the same.

\begin{restatable}{lemma}{typefeasibleequalsequence}
\label{lem:type-feasible-equal-sequence}
For any length $b \in \N$, there is a unique length $b$ $\sigma$-type feasible selection $\bS \subseteq \bV$ satisfying
\[
\multiset{ \type(e) \in [r]: e \in \bS } = \multiset{ q \in [r]: q \in \bs(b, b) }
\]
\end{restatable}
\begin{proof}
The claim trivially holds when $r = 1$, so we shall consider $r > 1$.
We will prove this by induction over $b \in \N$.

\hspace{0pt}\\
\noindent
\textbf{Base case} ($b = 1$):
Let $\bS$ be any type feasible solution of size $|\bS| = 1$ according to \cref{def:type-feasible-selection}.
Since $r > 1$, we have $\beta(1) = 0$, there exists some $q \in [r]$ such that $\bS \cap \bT_q = \emptyset$, and $|\bQ(\bS)| = r - 1$.
Since $\bS$ is type feasible, it has to break ties according to $\sigma$.
Thus, $\bS = \{ q \in [r] : \sigma(q) = 1 \}$ is the unique type feasible solution of size $1$.
Meanwhile, \cref{def:min-ratio-type-sequence} produces $\bs(1,1) = (q^*)$, where $\sigma(q^*) = 1$.
This is because all ratios are $0$ initially and ties are broken according to $\sigma$.
So, $\bS = \bs(1,1)$ as desired.

\hspace{0pt}\\
\noindent
\textbf{Inductive case} ($b > 1$):
Let $\bS_{b-1}$ be any type feasible solution of length $|\bS_{b-1}| = b-1$.
By induction hypothesis, $\bS_{b-1}$ is unique and $\bS_{b-1} = \bs(b-1, b-1)$.
By \cref{lem:min-ratio-type-sequence-is-independent-of-denominator}, we know that $\bs(b-1, b-1) = \bs(b-1, b)$.
Since $\bS_{b-1} = \bs(b-1, b)$, we have $|\bS_{b-1} \cap \bT_q| = \#(q, \bs(b-1, b))$ for all types $q \in [r]$.
By definition, both $\bs(b-1, b)$ will extend by choosing something from $\bQ(\bS_{b-1})$ and break ties according to $\sigma$.
Meanwhile, we know from \cref{lem:sigma-type-feasible-is-extendable} that $\bS_{b-1} \subset \bS_b$, so we can also obtain $\bS_b$ by extending $\bS_{b-1}$ by choosing something from $\bQ(\bS_{b-1})$ and break ties according to $\sigma$.
As such, both sequences will extend by appending the same type, we have $\bS = \bs(b,b)$ as desired.
\end{proof}

For convenience, we reproduce \OurAlg{} below, which was designed with \cref{def:min-ratio-type-sequence} in mind.
\cref{lem:typewise-our-output-equal-sequence} below tells us that any choice of independent set of size $b^{(t)}$ at each time step by \OurAlg{} results in the same multiset of types as the prefix of the min-ratio type sequence of length $\sum_{t=1}^h b^{(t)}$.
Note that \OurAlg{} may not choose the elements whose type sequence agrees exactly with \cref{def:min-ratio-type-sequence}, but the total counts will agree for any given budget.

\begin{algorithm}[htb]
\caption{\OurAlg{}}
\label{alg:ouralg-reproduced}
\begin{algorithmic}[1]
\Statex \textbf{Input}: An $\OurProblem{}(f, h, r, \bV, \bb, \bp)$ instance and a tie-breaking ordering $\sigma: [r] \to [r]$
\Statex \textbf{Output}: A $\sigma$-type feasible selection $\bS = \bS^{(1:h)}$
\State Define $\bS^{(1)} = \ldots = \bS^{(h)} = \emptyset$ \Comment{Initialize selections}
\For{$t = 1, \ldots, h$}
    \State For $q \in [r]$, define $\#(q, \bs(b^{(1:0)}, b^{(1:0)})) = 0$ and
    \[
    x^{(t)}_q = \#(q, \bs(b^{(1:t)}, b^{(1:t)})) - \#(q, \bs(b^{(1:t-1)}, b^{(1:t-1)}))
    \]
    \State Let $\cM^{(t)} = (\bV^{(t)}, \cI^{(t)})$ be the partition matroid at
    \Statex\hspace{\algorithmicindent}time $t$, where $\cI^{(t)} = \{ \bS \subseteq \bV^{(t)}: | \bS \cap \bT_q | \leq x^{(t)}_q \}$.
    \For{$b^{(t)}$ times} \Comment{$\sum_{q \in [r]} x^{(t)}_q = b^{(t)} \leq |\bV^{(t)}|$}
        \State Let $\bS = \bigcup_{t=1}^h \bS^{(t)}$ be the selection so far
        \State Let $e^* = \argmax\limits_{\substack{e \in \bV^{(t)} \setminus \bS\\ \bS^{(t)} \cup \{e\} \in \cI^{(t)}}} g(\bS, \{e\})$ to $\bS^{(t)}$
    \EndFor
\EndFor
\State Output $\bS^{(1)} \uplus \ldots \uplus \bS^{(h)}$
\end{algorithmic}
\end{algorithm}

\begin{restatable}{lemma}{typewiseouroutputequalsequence}
\label{lem:typewise-our-output-equal-sequence}
For any time step $t \in [h]$, let $b^{(1:t)} = \sum_{\tau=1}^t b^{(t)}$ and $\cM^{(t)} = (\bV^{(t)}, \cI^{(t)})$ be the matroid defined in \OurAlg{}.
Then, we have $\bm^{t}_1 = \bm^{t}_2$ for all $t \in [h]$, where multisets $\bm^{t}_1, \bm^{t}_2 \in [r]^{b^{(1:t)}}$ are defined as
\[
\bm^{t}_1 = \multiset{ q \in [r]: q \in \bs(\bp, \sigma, b^{(1:t)}, b^{(1:t)}) }
\]
and
\[
\bm^{t}_2 = \multiset{ \type(e) \in [r]: e \in \bigcup_{\tau=1}^t \bS^{(\tau)}, \text{ where}
\bS^{(\tau)} \in \cI^{(\tau)} \text{ and } |\bS^{(\tau)}| = b_\tau \quad \forall \tau \in [t]}
\]
\end{restatable}
\begin{proof}
We will prove this claim via an induction over $t \in \{1, \ldots, h\}$.

\hspace{0pt}\\
\noindent
\textbf{Base case} ($t=1$):
By definition of $\cM^{(1)}$, any independent set of size $b^{(1)}$ in $\cI^{(1)}$ involves $x^{(1)}_q = \#(q, \bs(b^{(1)}, b^{(1)}))$ items of type $q \in [r]$ such that $\sum_{q \in [r]} x^{(1)}_q = b^{(1)}$.
So, the multisets $\bm^{(t)}_2 = \bm^{(t)}_1$ both count the types appearing in $\bs(b_{1}, b_{1})$.

\hspace{0pt}\\
\noindent
\textbf{Inductive case} ($t > 1$):
By definition of $\cM^{(t)}$, any independent set of size $b^{(t)}$ in $\cI^{(t)}$ involves
\begin{align*}
x^{(t)}_q
&= \#(q, \bs(b^{(1:t)}, b^{(1:t)})) - \#(q, \bs(b^{(1:t-1)}, b^{(1:t-1)})) \tag{By definition of $x^{(t)}_q$ in \cref{alg:ouralg}}\\
&= \#(q, \bs(b^{(1:t)}, b^{(1:t)})) - \#(q, \bs(b^{(1:t-1)}, b^{(1:t)})) \tag{By \cref{lem:min-ratio-type-sequence-is-independent-of-denominator}}
\end{align*}
items of type $q \in [r]$ such that $\sum_{q \in [r]} x^{(t)}_q = b^{(t)}$.
Let $\bx^{(t)}$ be the multiset of size $b^{(t)}$ such that $q \in [r]$ appears $x^{(t)}_q$ times.
That is, the multiset $\bx^{(t)}$ counts the types appearing in the length $b^{(t)} = b^{(1:t)} - b^{(1:t-1)}$ suffix of $\bs(b^{(1:t)}, b^{(1:t)})$.

By \cref{lem:min-ratio-type-sequence-is-independent-of-denominator}, we also know that $\bs(b^{(1:t-1)}, b^{(1:t-1)}) = \bs(b^{(1:t-1)}, b^{(1:t)})$, so $\bs(b^{(1:t-1)}, b^{(1:t-1)})$ is simply the length $b^{(1:t-1)}$ prefix of $\bs(b^{(1:t)}, b^{(1:t)})$ according to \cref{def:min-ratio-type-sequence}.
By induction hypothesis, this means that $\bm^{(t-1)}_2$ is the multiset counting the types appearing in the length $b^{(1:t-1)}$ prefix of $\bs(b^{(1:t)}, b^{(1:t)})$.
So, the multisets $\bm^{(t)}_2 = \bm^{(t-1)}_2 \cup \bx^{(t)} = \bm^{(t)}_1$ both count the types appearing in $\bs(b^{(1:t)}, b^{(1:t)})$.
\end{proof}

We are now ready to prove our \cref{thm:apx-ratio-of-our-algo} which states the theoretical guarantees of \OurAlg{} with respect to $\sigma$-type feasible selections.

\apxratioofouralgo*
\begin{proof}
In each time step $t \in [h]$, \OurAlg{} can compute the $x^{(t)}_q$ values in $\cO(r \cdot b^{(t)})$ time by appropriate bookkeeping, then perform $b^{(t)}$ greedy selections, each requiring the calculation of at most $|\bV^{(t)}|$ marginal gains.
The overall running time across all time steps is $\cO(\sum_{t=1}^h b^{(t)} \cdot (r + |\bV^{(t)}|))$.

We now show the three items separately.

\medskip
\noindent
1. We always choose an independent set containing elements from $\bV^{(t)}$ at each time step $t \in [h]$, up to size $b^{(t)}$.

\medskip
\noindent
2. \cref{lem:type-feasible-equal-sequence} and \cref{lem:typewise-our-output-equal-sequence} jointly ensures that \emph{any} output of \OurAlg{} is a $\sigma$-type feasible selection.

\medskip
\noindent
3. \cref{lem:type-feasible-equal-sequence} and \cref{lem:typewise-our-output-equal-sequence} jointly ensures that \emph{any} $\sigma$-type feasible solution is a \emph{possible} output of \cref{alg:ouralg}.
As $\bV^{(1)}, \ldots, \bV^{(h)}$ are disjoint, $\cI_1, \ldots, \cI_h$ jointly induce a \emph{single} partition matroid which enforces $\sum_{q=1}^r x^{(t)}_q = b^{(t)}$ at each time step $t \in [h]$.
Since \OurAlg{} is performs greedy selection at each time step over a single matroid, it yields a $\frac{1}{2}$-approximation to the offline optimum selection at any time step $t \in [h]$; just like \textsc{LocalGreedy} with $k = 1$.
\end{proof}

\section{Deferred technical details for \texorpdfstring{\cref{sec:impossibility}}{Section 4.3}}

Here, we provide the full proofs of the impossibility results stated in \cref{sec:impossibility}.

\typefeasibleforfutureproof*
\begin{proof}
Consider an arbitrary \OurProblem{} instance $\cM$ with $r$ types for which there is some timestep $t$ such that $\ALG$ maximizes $\beta(b^{(1:t)})$ but did not minimize $|\bQ(\bS^{(t)})|$.
Suppose $|\bQ(\bS^{(t)})| = k \leq r$ while it could have been of size $\ell < k$.
We now define a new \OurProblem{} instance $\cM'$ obtained by using the prefix of $\cM$ up to time $t$, and adding one additional time step $b^{(t+1)} = \ell$.
Since $\cM$ and $\cM'$ are indistinguishable up to time step $t$, and $\ALG$ is deterministic, $\ALG$ must choose the same selections up to time step $t$, resulting in $|\bQ(\bS^{(t)})| = k > \ell$.
Thus, for the final time step $t+1$ for $\cM'$, $\ALG$ would have $\beta(b^{(1:t+1)}) = \beta(b^{(1:t)})$ instead of being able to improve the satisfaction ratio if it had minimized $|\bQ(\bS^{(t)})|$ at time step $t$ to be of size $\ell$.
\end{proof}

We prove \cref{prop:impossibility-tight} (reproduced below for convenience) separately for the bound of $\frac{1}{2}$ and $\frac{k}{k+1}$ in \cref{prop:half-ratio-is-tight} and \cref{prop:no-sigma-is-tight} respectively.

\impossibilitytight*
\begin{proof}
See \cref{prop:half-ratio-is-tight} and \cref{prop:no-sigma-is-tight}.
\end{proof}

\begin{restatable}{proposition}{halfratioistight}
\label{prop:half-ratio-is-tight}
There exists instances where \OurAlg{} achieves an approximation ratio of at most $\frac{1}{2}$ due to budget uncertainties.
\end{restatable}
\begin{proof}
Fix values $x, \eps \geq 0$, where $\eps$ is to be treated as a vanishingly small value.
Suppose $r = 1$, $h = 2$, $b^{(1)} = b^{(2)} = 1$, $\bV^{(1)} = \{a, b\}$, $\bV^{(2)} = \{c\}$, and
\[
f(\bS) =
\begin{cases}
0 & \text{if $\bS = \emptyset$}\\
x + \eps & \text{if $\bS = \{a\}$}\\
x & \text{if $\bS = \{b\}$}\\
x + 2 \eps & \text{if $\bS = \{c\}$ or $\bS = \{a,c\}$}\\
2x + 2 \eps & \text{if $\bS = \{b,c\}$}
\end{cases}
\]
One can verify that $f$ is a non-decreasing submodular function.
\OurAlg{} will greedily select $\bS^{(1)} = \{a\}$ and $\bS^{(2)} = \{c\}$, obtaining an objective value of $f(\{a,c\}) = x + 2 \eps$.
Meanwhile, an optimal selection that operates without budget uncertainty will select $\bS^{(1)} = \{b\}$ and $\bS^{(2)} = \{c\}$, obtaining an objective value of $f(\{a,c\}) = 2x + 2 \eps$.
We see that the competitive ratio is $\lim_{\eps \to 0} \frac{x + 2 \eps}{2x + 2 \eps} = \frac{1}{2}$.
\end{proof}

\begin{restatable}{proposition}{nosigmaistight}
\label{prop:no-sigma-is-tight}
There exists instances where \OurAlg{} achieves an approximation ratio of at most $\frac{k}{k+1}$ due to $\sigma$-ordering.
\end{restatable}
\begin{proof}
Fix $k \in \N$ and the ordering $\sigma: [r] \to [r]$ which favors smaller indexed types.
That is, when tie-breaking for satisfaction ratio, the unselected type(s) have large indices.
Suppose $|\OPT_{\sigma} \cap \bV^{(t)} \cap \bT_q| \geq k$ for all $t \in [h]$ and $q \in [r]$.
Now suppose $h = 1$, $b^{(1)} = rk + 1$, $f(\bS) = |\bS \cap \bT_r|$, and there are more than $b^{(1)}$ elements of each type.
One can check that that $f$ is a non-decreasing submodular function.
Then, \OurAlg{} will select $k$ elements of each type and then one element of type $1$ due to $\sigma$, obtaining an objective value of $k$.
Meanwhile, an optimal selection will select $k$ elements of each type and then one element of type $r$, obtaining an objective value of $k+1$.
Thus, the competitive ratio is $\frac{k}{k+1}$.
\end{proof}

\section{Results for Sidama Region}
\label{sec:appendix-experiment-sidama}

\subsection{Experiment 1: Impact of budget on coverage}
In \cref{fig:budgetapp}, the bars show how total coverage improves under different allocation strategies and budget levels
over a five-year planning horizon. Since the initial coverage is already high, the measured improvement is marginal. 

\begin{figure*}[htb]
    \centering
    \includegraphics[width=1.0\linewidth]{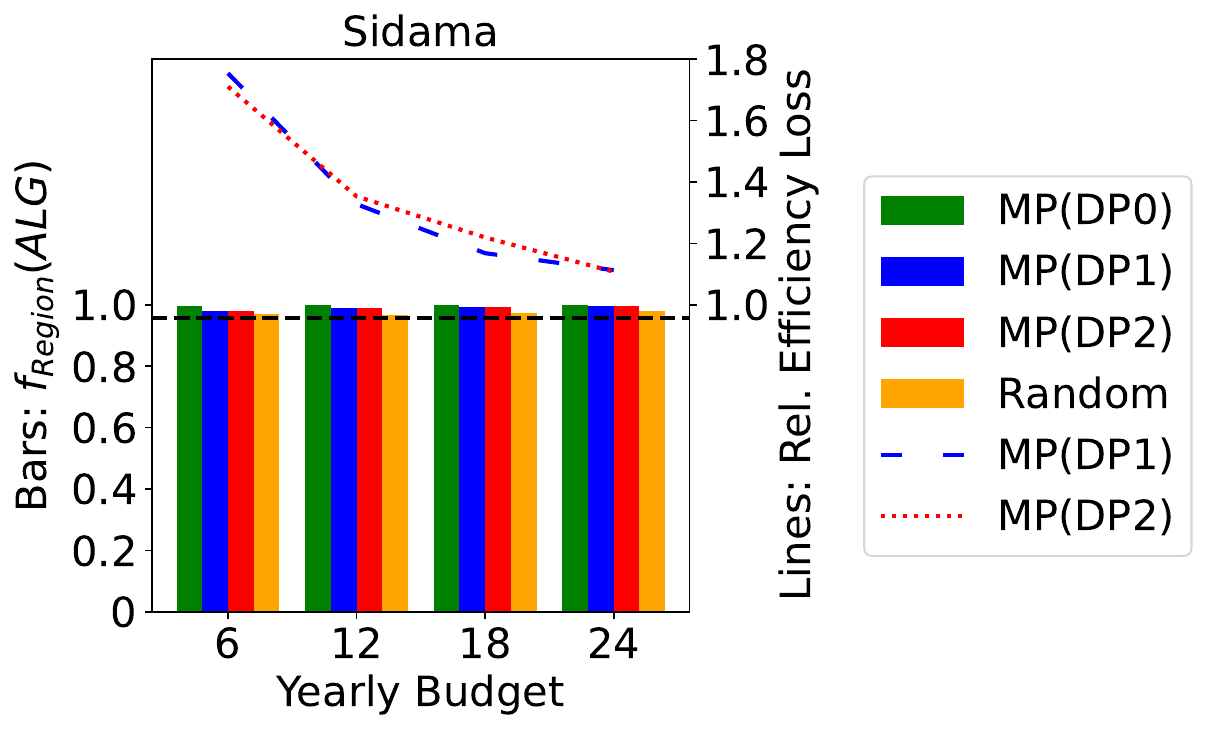}
    \caption{Coverage under varying annual budgets for Sidama region. Bars show total population coverage by each policy, with the black line indicating existing coverage before any yearly budget was spent to build additional health posts. The dashed line above indicate relative efficiency loss from enforcing proportional constraints for DP$1$ and DP$2$, compared to the unconstrained baseline (DP$0$).}
    \label{fig:budgetapp}
\end{figure*}

\subsection{Experiment 2: Evaluating district-level equity}
\begin{figure}[htb]
    \centering
    \includegraphics[width=0.9\columnwidth]{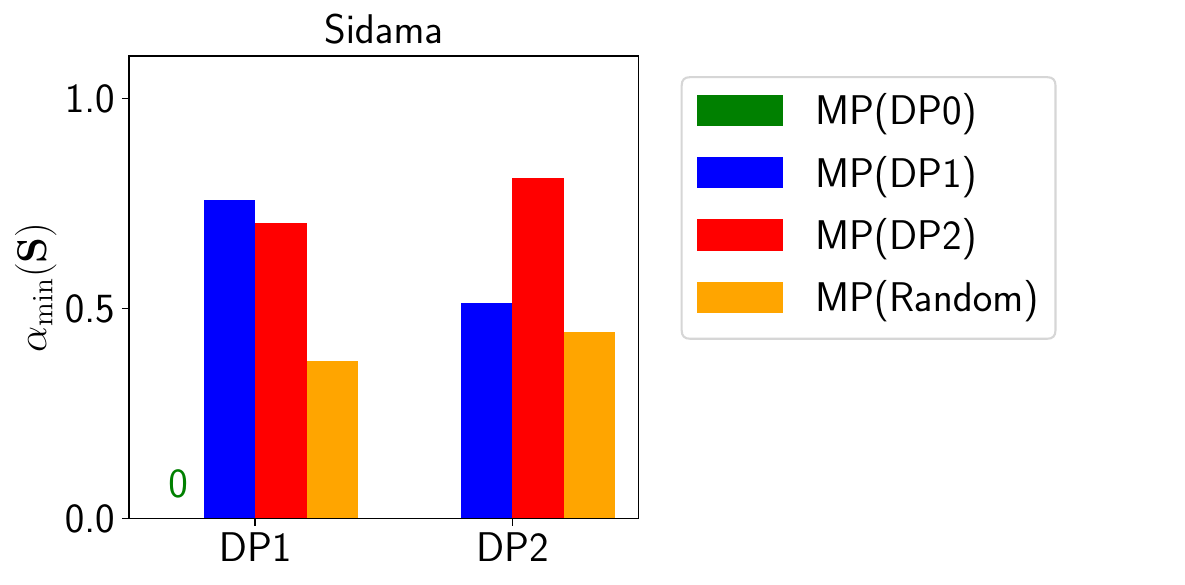}
    \caption{Minimum satisfaction ratio $\alpha_{\min}$ attained under each policy, with proportions defined by DP$1$ or DP$2$ respectively. Over five-year horizon, annual budgets of $12$ facilities were allocated for Sidama.}
\label{fig:typesappendix}
\end{figure}

Here we examine the minimum satisfaction ratio, $\alpha_{\min}(\bS)$, achieved under each policy.
As shown in \cref{fig:typesappendix}, \textsc{MP}(DP$1$) and \textsc{MP}(DP$2$) outperform all other methods in the DP$1$ and DP$2$ metrics respectively, while \textsc{MP}(DP$0$) consistently attains $\alpha_{\min} = 0$.
Interestingly, although in~\cref{fig:budgetapp}, building facilities resulted in mild effect on the coverage, the effect on the district-level equity is more noticeable.

\subsection{Experiment 3: High quality advice can help}

The results in this section are similar to the results in the main paper.
Out of around $30$ districts in Sidama, we found $11$ districts in which our learning-augmented selection $\bU$ outperforms both the greedy baseline
$\bG$ and the expert selection $\bA$, i.e., $f(\bU) > f(\bG) > f(\bA)$.

\section{Further experimental details}
\label{sec:appendix-experiment}

All experiments were conducted using Python 3.12 on two systems:
(i) a Linux-based server with 48-core Intel Xeon CPUs and 190GB of RAM, and
(ii) a Windows laptop running WSL with 8 cores and 16GB of RAM.

Population projections were obtained from the WorldPop dataset\footnote{\url{https://data.worldpop.org/GIS/Population/Global_2015_2030/R2024B/2026/ETH/v1/1km_ua/constrained/eth_pop_2026_CN_1km_R2024B_UA_v1.tif}}.
Friction surface data for walking time computation was sourced from the MalariaAtlas GitHub repository\footnote{\url{https://github.com/malaria-atlas-project/malariaAtlas}}, using an R script we provide.
Administrative shapefiles were downloaded from the United Nations Office for the Coordination of Humanitarian Affairs (OCHA)\footnote{
\url{https://data.humdata.org/dataset/cb58fa1f-687d-4cac-81a7-655ab1efb2d0/resource/63c4a9af-53a7-455b-a4d2-adcc22b48d28/download/eth_adm_csa_bofedb_2021_shp.zip}}.

For Experiment $3$ (retrospective analysis), we fixed a random seed of $0$, and generated $10$ permutations of the original advice set to serve as inputs to our learning-augmented algorithm.

Several data inputs (such as district-level rates of home births and postnatal care, as well as facility registries) request permission and thus cannot  be shared.
To ensure usability, we include code that generates synthetic examples, allowing others to replicate our methods.

\paragraph{Reproducibility focus.}
Our goal is to make our methodological contribution reproducible and reusable.
All core components of our tool are fully data-agnostic and will yield the same visualizations and planning outputs when provided with actual Ethiopia-specific data. 
Once these data become publicly available, all results in the main paper can be replicated by plugging them into our codebase.
The code is available at \url{https://github.com/yohayt/OHCE/}.

\end{document}